\documentclass{article}

\PassOptionsToPackage{numbers, compress}{natbib}
\usepackage[preprint]{neurips_2026}

\usepackage[utf8]{inputenc} 
\usepackage[T1]{fontenc}    
\usepackage{hyperref}       
\usepackage{url}            
\usepackage{booktabs}       
\usepackage{amsfonts}       
\usepackage{nicefrac}       
\usepackage{microtype}      
\usepackage{xcolor}         

\usepackage{amsthm}

\theoremstyle{plain}
\newtheorem{theorem}{Theorem}[section]
\newtheorem{proposition}[theorem]{Proposition}
\newtheorem{lemma}[theorem]{Lemma}
\newtheorem{corollary}[theorem]{Corollary}
\theoremstyle{definition}
\newtheorem{definition}[theorem]{Definition}
\newtheorem{assumption}[theorem]{Assumption}
\newtheorem{example}[theorem]{Example}
\theoremstyle{remark}
\newtheorem{remark}[theorem]{Remark}

\newcommand{\qone}{\textcolor{purple}{\textbf{Q1}}}
\newcommand{\qtwo}{\textcolor{purple}{\textbf{Q2}}}
\newcommand{\qthree}{\textcolor{purple}{\textbf{Q3}}}

\usepackage[disable,textsize=tiny]{todonotes}
\usepackage{enumitem}
\usepackage{thmtools} 
\usepackage{thm-restate}
\usepackage{latexsym}
\usepackage{amsmath}
\usepackage{amssymb}
\usepackage{bm}
\usepackage{mathrsfs}
\usepackage{url}

\newcommand{\ab}{\bm{a}}
\newcommand{\sbb}{\bm{s}}

\newcommand{\Acal}{\mathcal{A}}

\newcommand{\Lcal}{\mathcal{L}}

\newcommand{\Scal}{{\mathcal{S}}}

\newcommand{\EE}{\mathbb{E}} 
\newcommand{\RR}{\mathbb{R}} 
\newcommand*{\argmax}{\mathop{\mathrm{argmax}}}

\ifx\BlackBox\undefined
\newcommand{\BlackBox}{\rule{1.5ex}{1.5ex}}  
\fi

\ifx\QED\undefined
\def\QED{~\rule[-1pt]{5pt}{5pt}\par\medskip}
\fi

\ifx\proof\undefined
\newenvironment{proof}{\par\noindent{\bf Proof\ }}{\hfill\BlackBox\\[2mm]}
\fi

\ifx\theorem\undefined
\newtheorem{theorem}{Theorem}

\newtheorem{lemma}[theorem]{Lemma}
\newtheorem{proposition}[theorem]{Proposition}
\newtheorem{remark}[theorem]{Remark}

\fi

\ifx\axiom\undefined
\newtheorem{axiom}[theorem]{Axiom}
\fi

\newcommand{\ie}{\emph{i.e.}}

\newcommand{\rbr}[1]{\left(#1\right)}

\renewcommand{\url}[1]{{\sffamily #1}}

\renewcommand{\leq}{\leqslant}
\renewcommand{\le}{\leqslant}
\renewcommand{\geq}{\geqslant}
\renewcommand{\ge}{\geqslant}

\newcommand{\secref}[1]{Section~\ref{#1}}

\newcommand{\appref}[1]{Appendix~\ref{#1}}

\definecolor{highlightgray}{gray}{0.93}
\definecolor{sippo_std}{HTML}{F6E1C4} 
\definecolor{sippo_neg}{HTML}{D1E2F3} 

\usepackage{amsmath}
\usepackage{algorithm}
\usepackage{algorithmic}
\usepackage{listings}
\usepackage{xcolor}
\usepackage{amssymb}
\usepackage{mathtools}
\usepackage{amsthm}
\usepackage{booktabs}
\usepackage{multirow}
\usepackage{xspace}
\usepackage{siunitx}    
\usepackage[table]{xcolor} 
\usepackage{wrapfig}
\usepackage{enumitem}
\usepackage{xspace}

\usepackage{xcolor}
\usepackage[framemethod=TikZ]{mdframed}
\usepackage{amsthm}

\newtheorem{exmp}{Example}[section]

\newmdenv[
  backgroundcolor=gray!10,
  linecolor=gray!40,
  linewidth=0.5pt,
  innertopmargin=4pt,
  innerbottommargin=2pt,
  innerleftmargin=4pt,
  innerrightmargin=4pt,
  roundcorner=2pt,
]{shadedbox}

\newenvironment{myexample}
  {\begin{shadedbox}\begin{exmp}}
  {\end{exmp}\end{shadedbox}}

\newcommand{\boldcell}[2]{%
  \bfseries $\mathbf{#1} \pm #2$
}

\newcommand{\underlinecell}[2]{%
  \bfseries $\underline{#1} \pm #2$
}

\usepackage[capitalize,noabbrev]{cleveref}

\title{GeMPO: Generalized Measure Matching for Online Diffusion Reinforcement Learning}

\author{%
  Haitong Ma\thanks{Equal contribution.} \\
  Harvard University \\
  \And
  Chenxiao Gao\footnotemark[1] \\
  Georgia Institute of Technology \\
  \And
  Tianyi Chen \\
  Georgia Institute of Technology \\
  \AND
  Na Li \\
  Harvard University \\
  \And
  Bo Dai \\
  Georgia Institute of Technology \\
}

\begin{document}

\maketitle

\newcommand{\algbb}{GeMPO\xspace}

\begin{abstract}

A commonly used family of RL algorithms for diffusion policies conducts softmax reweighting over samples from the behavior policy, which often induces an over-greedy policy and fails to utilize feedback from negative samples. In this work, we introduce \algbb, a simple and unified framework that generalizes reweighting scheme in diffusion RL from softmax to general monotonic functions. \algbb revisits diffusion RL via a measure matching perspective: \textit{First}, we construct a virtual target policy measure via solving a regularized policy optimization objective; \textit{Second}, we minimize the divergence between the current policy and this target measure through reweighted flow matching. This formulation offers two key advantages: \textbf{i)} It extends weight design beyond traditional exponential reweighting, allowing it to be tailored to diverse reward landscapes; and \textbf{ii)} by relaxing the non-negativity constraint on the target measure, our framework provides a principled justification for negative reweighting. We provide interpretations of how negative reweighting actively repels the policy from suboptimal actions and thus facilitates exploration. Extensive empirical evaluations demonstrate that \algbb achieves competitive or superior performance by leveraging these flexible weighting schemes, and we provide practical guidelines for selecting reweighting methods in practice.
\end{abstract}

\vspace{-4mm}
\section{Introduction}

Diffusion models~\cite{ho2020denoising,song2019generative} and flow models~\cite{liu2022flow,lipman2022flow} have emerged as a dominant class of generative models for synthesizing high-fidelity data across various modalities, from photorealistic imagery to complex structural biology. Beyond traditional density estimation, there is an increasing demand to align these models with specific downstream objectives, such as human preferences~\cite{rafailov2023direct}, biological evaluations~\cite{avsec2021effective}, and real-world physical feedback~\cite{lei2025rl}. Consequently, Reinforcement Learning (RL) for diffusion models has become a critical research direction to achieve the downstream alignment.   

The existing RL algorithms for diffusion models generally fall into two categories~\cite{gao2026flowrl}. One prominent direction treats the denoising reverse process as a Markov Decision Process (MDP) and optimizes the model using policy gradient~\cite{black2023training,fan2023dpok,ren2024diffusion,liu2025flow,zhang2025reinflow} or reparameterization~\cite{wang2024diffusion,wang2025enhanced,dime}, which are computationally intensive. Furthermore, they usually mandate the use of Stochastic Differential Equation (SDE) samplers during training, which may not only require a very different training infrastructure from deployment, but also compromises the deployment performance since the sampler may differ. 

An alternative RL paradigm, also known as advantage weighted regression (AWR) \citep{peng2019advantage,nair2020awac}, reinforces the current policy by regressing toward behavior samples using weights that are exponential to their expected returns. Extending this concept to diffusion RL, \citet{ma2025efficient} introduced a suite of efficient algorithms that solve \textit{reweighted flow matching} problems, maintaining high compatibility with existing diffusion model training infrastructures. However, this paradigm tends to assign disproportionately high weights to a small subset of high-advantage samples while assigning negligible weights to the remainder, leading to over-greedy policy updates~\citep{mei2020escaping} and a failure to leverage information from negative samples~\cite{frans2025diffusion}. 

To mitigate the inherent limitations of exponential reweighting, we introduce Generalized Measure Matching for Diffusion Policy Optimization (\algbb), a unified framework that significantly broadens the scope of reweighted flow matching through flexible weighting function design. We reformulate the derivation of reweighted flow matching into two logical stages: \textit{First}, a target policy measure is constructed by solving a regularized RL objective; \textit{Second}, we align the current policy with the target measure through reweighted flow matching. This perspective enables \algbb to not only subsume existing weighting designs~\citep{ding2024diffusion,ma2025efficient,tang2025wd1} as special cases but also to decouple the framework from rigid exponential constraints, allowing the reweighting to be tailored to specific reward landscapes. Furthermore, by relaxing the standard non-negativity constraint on the target measure, \algbb can accommodate a broad class of monotone reweighting functions, even the ones with negative weights. We provide interpretations on how this mechanism actively repels the policy from suboptimal regions, addressing a critical drawback of conventional designs. Extensive evaluations across multi-armed bandits, MuJoCo locomotion, and DNA sequence generation demonstrate that \algbb achieves competitive or superior performance by leveraging these flexible reweighting schemes.

Due to space limit, more related works is presented in Appendix~\ref{sec.related_works}.

\section{Preliminaries}
\label{sec.preliminaries}
\paragraph{KL-regularized Policy Optimization.}We consider the Markov Decision Process (MDP)~\citep{puterman2014markov} specified by a tuple $\mathcal{M}=(\mathcal{S}, \mathcal{A}, r, P, \mu_0, \gamma)$, where $\mathcal{S}$ is the state space, $\mathcal{A}$ is the action space, $r:\Scal\times\Acal\to\RR$ is a reward function,
$P: \mathcal{S} \times \mathcal{A} \rightarrow \Delta(\mathcal{S})$ is the transition operator, $\mu_0 \in \Delta(\mathcal{S})$ is the initial distribution and $\gamma \in(0,1)$ is the discount factor. We define $Q^{\pi}(\sbb, \ab)=\EE_{\pi}[\sum_{t=0}^\infty \gamma^t r(\sbb_t,\ab_t)|\sbb_0=s,\ab_0=a]$ as the state-action value function. Consider optimizing a policy $\pi:\mathcal{S}\to\Delta(\mathcal{A})$, via the KL-regularized objective function,
\begin{equation}\label{eq:pmd}
\pi = \argmax_{\pi\in\Pi} \EE_{\ab \sim \pi} [Q^{\pi_{\text{old}}}(\sbb, \ab)] - \lambda D_{\mathrm{KL}}(\pi \| \pi_{\text{old}}; \sbb),
\end{equation}
where $\Pi\doteq\{\pi(\cdot|s)|\int\pi(\cdot|s)=1,\pi(\cdot|s)\geq 0,\forall s\in\mathcal{S}\}$ is the policy family mapping states to probability distributions on the action space. $\pi_{\text{old}}$ is the current policy, $\lambda\in\mathbb{R}^+$ is the regularization strength, and the Kullback–Leibler (KL) divergence $D_{\rm KL}(p||q) = \mathbb{E}_{x\sim p}[\log\frac{p(x)}{q(x)}]$ serves to constrain the updated policy within a trust region around the old policy $\pi_{\text{old}}$, similar to proximity-based algorithms like TRPO~\cite{schulman2015trust} and PPO~\cite{schulman2017proximal}. 
The closed-form optimal solution in policy family $\Pi$ to this objective is given by an exponential reweighting of the current policy:
\begin{equation}
\pi^*(\ab|\sbb) \propto \pi_{\text{old}}(\ab|\sbb) \exp(Q^{\pi_{\text{old}}}(\sbb, \ab) / \lambda). \label{eq.softmax_reweighting}
\end{equation}


\paragraph{Diffusion and Flow Models.}
Diffusion models~\cite{ho2020denoising,song2019generative} or flow models~\cite{lipman2022flow,liu2022flow} learn data distributions by gradually perturbing clean data $\boldsymbol{x}_0 \sim p_{\text {data }}$ with Gaussian noise according to a forward process:
$$
\boldsymbol{x}_t=\alpha_t \boldsymbol{x}_0+\sigma_t \boldsymbol{\epsilon}, \boldsymbol{\epsilon} \sim \mathcal{N}(\mathbf{0}, \mathbf{I}),
$$
which emits the conditional distribution $q_{t|0}(\boldsymbol{x}_t|\boldsymbol{x}_0)=\mathcal{N}(\boldsymbol{x}_t; \alpha_t\boldsymbol{x}_t, \sigma_t^2)$. It then learns the score or velocity function to reverse this process to sample from the target distribution $p_{\text{data}}$. We employ the $\boldsymbol{v}$-prediction formulation~\citep{lipman2022flow,lipman2024flow} in this paper, where the velocity $\boldsymbol{v}$ is defined as the time-derivative of $\boldsymbol{x}$, \ie, $\boldsymbol{v}=\dot{\alpha}_t\boldsymbol{x}_0+\dot{\sigma}_t\boldsymbol{\epsilon}$. They fit a parameterized network $D_\theta(\boldsymbol{x}_t, t)$ to match the \textit{marginal} velocity field:
\begin{equation}\label{eq:unweighted_closed_form_velocity}
\boldsymbol{v}_t(\boldsymbol{x}_t)=\mathbb{E}_{\boldsymbol{x}_0\sim p_{0|t}(\cdot|\boldsymbol{x}_t)}\left[\boldsymbol{v}_{t|0}(\boldsymbol{x}_t|\boldsymbol{x}_0)\right],
\end{equation}
where $p_{0|t}(\boldsymbol{x}_0|\boldsymbol{x}_t)\doteq\frac{q_{t|0}(\boldsymbol{x}_t|\boldsymbol{x}_0)p_{\rm data}(\boldsymbol{x}_0)}{\int q_{t|0}(\boldsymbol{x}_t|\boldsymbol{x}_0)p_{\rm data}(\boldsymbol{x}_0)d\boldsymbol{x}_0}$ is the posterior distribution, $\boldsymbol{v}_{t|0}(\boldsymbol{x}_t|\boldsymbol{x}_0)$ is the \textit{conditional} velocity field:
$$
\boldsymbol{v}_{t|0}(\boldsymbol{x}_t|\boldsymbol{x}_0)=\frac{(\sigma_t\dot{\alpha}_t-\dot{\sigma}_t\alpha_t)}{\sigma_t}\boldsymbol{x}_0+\frac{\dot{\sigma}_t}{\sigma_t}\boldsymbol{x}_t.
$$
However, the marginal velocity field is not accessible due to the intractability of the posterior distribution, therefore in practice, we employ the conditional velocity field matching, which produces equivalent gradients as matching the marginal velocity field:
$$
\mathcal{L}(\theta)=\mathbb{E}_{\boldsymbol{x}_0\sim p_{\rm data}, \boldsymbol{\epsilon}\sim\mathcal{N}}\left[\|D_\theta(\boldsymbol{x}_t, t)-\boldsymbol{v}_{t|0}\|^2\right].
$$
After training is complete, sampling is done by solving a stochastic or ordinary differential equation (SDE/ODE):
\begin{equation}
    \begin{aligned}
\mathrm{d}\boldsymbol{x}_t = \left( (1+\eta^2)\boldsymbol{v}_t(\boldsymbol{x}_t) - \eta^2\frac{\dot{\alpha}_t}{\alpha_t}\boldsymbol{x}_t \right) \mathrm{d}t + \eta\sqrt{2\sigma_t\dot{\sigma}_t - 2\frac{\dot{\alpha}_t}{\alpha_t}\sigma_t^2}\mathrm{d}\bar{\boldsymbol{w}}_t,
    \end{aligned}
\end{equation}
where $\eta\geq 0$ controls the stochasticity. Different specifications of the noise schedule $\{(\alpha_t, \sigma_t)\}$ lead to different methods. For example, the variance-exploding (VE) diffusion model~\cite{song2019generative} has the schedule $\alpha_t = 1, \sigma_t = \sqrt{t}$. 
Flow models~\cite{lipman2022flow,liu2022flow} employ linear interpolation with $\alpha_t=1-t, \sigma_t=t$, which simplifies the conditional velocity target to $\boldsymbol{v}_{t|0}=\boldsymbol{\epsilon}-\boldsymbol{x}_0$. Note that diffusion models may also employ $\boldsymbol{x}_0$- or $\boldsymbol{\epsilon}$-prediction, which are equivalent to $\boldsymbol{v}$-prediction used in this paper due to the reparameterization among the three \citep{li2025back}.

\paragraph{$f$-divergence}~describes the difference between two probability distributions~\citep{csiszar2004information,liese2006divergences,ali1966general}. For a convex function $f: \mathbb{R}^{+} \rightarrow \mathbb{R}$ with $f(1)=0$, the corresponding $f$-divergence for two distributions $p, q$ is defined as:
\begin{equation}
    \textstyle
    D_f(p \| q)=\mathbb{E}_{x\sim q}\left[f\left(\frac{p(x)}{q(x)}\right)\right],    \label{eq:f-div}
\end{equation}
where $f$ is called the generator function. 
Different choices of $f$-divergence can cover a wide class of popular divergences, which will be explained later in \Cref{tab:f-divergence} in~\Cref{sec.f-div}.
%


%

\paragraph{Signed Measures and Notations.} $p, q$ in~\eqref{eq:f-div} are probability distributions, \ie, non-negative measures with total unit mass. Signed measures extend this notion, which may 
assign negative mass to some sets, provided the total magnitude of its mass is finite. 
In contrast to $\Delta(\mathcal{A})$ as the space of probability measures over the action
space $\mathcal{A}$, we let $\mathcal{M}_{\pm}(\mathcal{A})$ denote the space of signed measures over $\mathcal{A}$. 
The definition~\eqref{eq:f-div} can be extended to cases where $p$ is a finite signed measure~\cite{csiszar2002mem,broniatowski2006minimization} and $q$ remains a probability measure. This extension is well-defined under mild conditions\footnote{Such as $f$ being a measurable function on $\mathbb{R}$. Detailed derivations are provided in Appendix~\ref{sec.f_div_signed_measure}.}. 

\section{\algbb: Revisit Diffusion RL from Generalized Measure Matching}

In this section, we revisit diffusion RL from a measure matching perspective. Specifically, diffusion RL can be understood as first constructing the target policy measure which solves the regularized RL objective, and then matching the current policy with the target measure (\secref{sec.two-stage}). This naturally leads to a reweighted flow matching objective, where the weighting function depends on the formulation of the target measure. Instantiated with various $f$-divergence metrics, this framework subsumes existing algorithms as special cases while also inspiring novel weighting schemes (\secref{sec.f-div}). Furthermore, by relaxing the standard non-negativity constraint on the target measure, we introduce negative reweighting, which repels the policy from suboptimal actions and facilitates exploration (\secref{sec.negative-weighting}).

\subsection{Diffusion RL via Measure Matching}
\label{sec.two-stage}

We consider a generalized version of the regularized policy optimization problem \eqref{eq:pmd}:
\begin{equation}\label{eq.problem_formulation}
    \begin{aligned}
        \max_{\pi\in\Pi}~& \EE_{\pi}[Q(\sbb, \ab)] -\lambda  D_f(\pi\|\pi_{\rm old}).
    \end{aligned}
\end{equation}
Here, we assume $f:\RR^+\to \RR$ is strictly convex and differentiable and $f(1)=0$. Recall that the policy family $\Pi $ here enforces the normalization constraint $\int \pi(\ab |\sbb)\mathrm{d}\ab = 1$ and the non-negativity constraint $\pi(\ab |\sbb) \ge 0$ for all $\sbb \in \Scal$ and $\ab\in\Acal$. 

\begin{proposition}[Optimal solution to~\eqref{eq.problem_formulation} in $\Pi$.]
    The optimal policy admits the form
\begin{equation}\label{eq.optimal_solution_f_divergence}
\pi^*(\ab | \sbb) = \pi_{\rm old}(\ab | \sbb){g}_f\left(\frac{Q(\sbb, \ab) - \nu(\sbb)}{\lambda}\right),
\end{equation}
where $\nu:\mathcal{S}\to \mathbb{R}$ are multipliers whose actual values depend on $\lambda$ and $Q$, and ${g}_f$ is the \emph{clipped inverse derivative function} satisfying ${g}_f(x)=(f')^{-1}(x)\text{ if }x>f'(0)\text{ otherwise } 0$.
\end{proposition}
The detailed derivation is deferred to Appendix~\ref{sec.apdx.derivation.optimal_solution}. 

The next step is to match $\pi^*$ using diffusion-parameterized policy $\pi_\theta\in\Pi_\theta\subseteq \Pi$, where $\Pi_\theta$ is policy family parameterized by diffusion models $D_\theta$ introduced in~\Cref{sec.preliminaries}.
A simple and commonly used technique is reweighted flow matching \citep{ding2024diffusion,ma2025efficient}:
\begin{equation}\label{eq:weighted_matching}
    \mathcal{L}(\theta)=\mathbb{E}_{\substack{\sbb\sim\mathcal{B}\\\ab_0\sim\pi_{\rm old}\\\epsilon\sim\mathcal{N}}}\left[w(\sbb, \ab_0)\|D_\theta(\sbb,\ab_t, t) - \boldsymbol{v}_{t|0}\|^2\right], \text{where}~w(\sbb, \ab_0)\doteq{g}_f\left(\frac{Q(\sbb, \ab_0) - \nu(\sbb)}{\lambda}\right).
\end{equation}
where $\mathcal{B}$ is the replay buffer, $\ab_t=\alpha_t\ab_0+\sigma_t\epsilon$, and $w(\sbb, \ab_0)$ is the reweighting factor in \eqref{eq.optimal_solution_f_divergence}. 
In Appendix~\ref{apdx.closed_form}, we show that reweighted flow matching recovers the velocity field of the target policy measure, and the optimal diffusion parameterization $D^*_\theta$ with respect to~\eqref{eq:weighted_matching} is a weighted average of conditional velocities under the posterior distribution $p_{0 \mid t}(\cdot|\ab_t)$:
\begin{equation}
    D^*_\theta(\sbb,\ab_t, t) = \frac{\mathbb{E}_{\ab_0 \sim p_{0|t}} [ w(\sbb,\ab_0) \boldsymbol{v}_{t|0}(\ab_t|\ab_0) ]}{\mathbb{E}_{\ab_0 \sim p_{0|t}} [ w(\sbb,\ab_0) ]}\label{eq.closed_form_solution_reweighting}
\end{equation}
Compared to the unweighted formulation \eqref{eq:unweighted_closed_form_velocity}, this solution reweights the conditional velocity according to its value, thereby emphasizing actions with higher advantage. 

Beyond reweighted flow matching, there exist other methods to align the current policy with the target measure, such as contrastive learning. We adopt reweighted flow matching for its simplicity and discuss alternative options in \appref{app.nce}.



\subsection{Unifying Existing Reweighting Schemes via $f$-divergences}
\label{sec.f-div}

\begin{table}[h]
    \centering
\small{
    \caption{Examples of $f$-divergence, their generator functions, and resulting reweighting factors.}
    \label{tab:f-divergence}
    \centering
    \small{
    \begin{tabular}{llllc}
    \toprule
 Divergence $D_f$ & Generator $f(x)$ & reweighting factor $g_f(\cdot)$ \\
 \midrule
 Forward KL Divergence & $x \log x$ & $\exp(x)$\\
 $\chi^2$-Divergence & $x^2-1$ & $x$ \\
 $\alpha$-Divergence ($\alpha>1$) & $\frac{x^\alpha-x}{\alpha(\alpha-1)}$ & $x^{\frac{1}{\alpha - 1}}$  \\
 \bottomrule
\end{tabular}
}
}
\end{table}

By choosing the $f$-divergence as the forward KL divergence, we recover the exponential weighting scheme used in AWR \citep{peng2019advantage} and DPMD \citep{ma2025efficient}.
\begin{myexample}{(DPMD \citep{ma2025efficient} as forward KL-regularization)}
\label{example.kl}
Let $D_f=D_{\rm KL}$, therefore $f(x)=x\log x - x + 1$, $\bar{g}_f(x)= \exp(x)$, and $f'(0)=-\infty$. Recognizing $\nu(\sbb)$ as the normalization factor $Z(\sbb)$, we can recover DPMD with \eqref{eq:weighted_matching}:
$$
\begin{aligned}
    &\mathcal{L}_{\rm DPMD}(\theta)=\mathbb{E}_{\sbb, \ab_0,\epsilon}\left[\frac{\exp(Q(\sbb, \ab))}{Z(\sbb)}\Big\|D_\theta(\sbb,\ab_t, t) - \boldsymbol{v}_{t|0}\Big\|^2\right].
\end{aligned}
$$
\end{myexample}

Moreover, we can see QVPO~\cite{ding2024diffusion}, originally derived from a lower bound of policy gradient objective, as a special case of specifying $D_f$ as $\chi^2$ divergence.

\begin{myexample}{(QVPO~\cite{ding2024diffusion} as $\chi^2$ regularization)}
\label{example.relu_linear}
Let $D_f$ be the $\chi^2$-divergence by setting $f(x)=x^2-1$, $\bar{g}_f(x)=\frac{1}{2}x$, and $f'(0)=0$. The QVPO objective
$$
\begin{aligned}
    \mathcal{L}_{\rm QVPO}(\theta)=\mathbb{E}_{\sbb, \ab_0,\epsilon}\left[[Q(\sbb, \ab)- V(\sbb)]_+\|D_\theta(\sbb,\ab_t, t) - \boldsymbol{v}_{t|0}\|^2\right]
\end{aligned}
$$
can be recovered by \eqref{eq:weighted_matching} through fixing $\lambda=\frac 12$ and  $[\cdot]_+=\max(\cdot, 0)$.
\end{myexample}

By recognizing KL and $\chi^2$ divergences as specific instantiations of the $\alpha$-divergences, we can extend our framework to a more generalized regularization family. 

\begin{myexample}{(Power function reweighting as $\alpha$-divergence regularization)}
For general $\alpha$-divergences with $\alpha > 1$, we have $f(x) = \frac{x^\alpha-1}{(\alpha-1)\alpha}$, which induces $f'(0)=0$ and $\bar{g}_f(x)=(\alpha-1)x^{\frac{1}{\alpha-1}}$. The target policy can be therefore defined as
$$
\pi^*(\ab|\sbb)\propto \pi_{\rm old}(\ab|\sbb)\left[\frac{Q(\sbb, \ab)-\nu(\sbb)}{\lambda}\right]_+^{\frac{1}{\alpha-1}},
$$
where certain constant factors are absorbed into $\lambda$. This suggests the following objective:
$$
\begin{aligned}
    &\mathcal{L}_{\alpha}(\theta)=\mathbb{E}_{\sbb, \ab_0,\epsilon}\left[\left[\frac{Q(\sbb, \ab)-\nu(\sbb)}{\lambda}\right]_+^{\frac{1}{\alpha-1}}\Big\|D_\theta(\sbb,\ab_t, t) - \boldsymbol{v}_{t|0}\Big\|^2\right].
\end{aligned}
$$
\end{myexample}

By tuning $\alpha$, we can shape the curvature of the weighting function to better match the characteristics of the underlying value landscape. With an appropriate choice of $\alpha$, our approach can avoid both the overly flat weighting used in QVPO (Example \ref{example.relu_linear}) and the statistical instabilities induced by the exponential weighting in DPMD (Example \ref{example.kl}).



\subsection{Extension to Negative Reweighting}
\label{sec.negative-weighting}

Current policy family $\Pi$ strictly enforces the non-negativity constraint $\pi^*(\ab|\sbb) \geq 0$, which consequently forces the weights $g_f(\cdot)$ or $w(\cdot, \cdot)$ to be non-negative. In practice, this assigns infinitesimal or zero weight to suboptimal samples. As a result, these methods fail to exploit negative samples during training and are prone to getting trapped in local optima. 

Our measure matching perspective relaxes the policy family in~\eqref{eq.problem_formulation} to mapping from states to normalized signed measures $\widetilde{\Pi}\doteq\{\tilde{\pi}:\mathcal{S\to\mathcal{M}_{\pm}(\mathcal{A})}|\int\tilde{\pi}(a|s)=1,\forall s\}\supseteq \Pi$. 

\begin{proposition}[Optimal solution to~\eqref{eq.problem_formulation} in $\widetilde{\Pi}$]
    We extend $D_f$ in~\eqref{eq.problem_formulation} to allow $p$ being signed measures and policy family to be $\widetilde{\Pi}$. Then the optimal solution $\tilde\pi^*\in \widetilde{\Pi}$ to~\eqref{eq.problem_formulation} is: 
\begin{equation}\label{eq.optimal_solution_relaxed}
\begin{aligned}
    \widetilde{\pi}^*(\ab | \sbb) &\propto \pi_{\rm old}(\ab | \sbb)\tilde g_f\left(\frac{Q(\sbb, \ab) - \nu(\sbb)}{\lambda}\right),\\
\end{aligned}
\end{equation}
where $\tilde g_f(\cdot)$ is a monotonic increasing function and \textbf{not required to be non-negative}. The derivation is deferred to \Cref{sec.apdx.derivation.optimal_solution_signed}.
\end{proposition}
 
Correspondingly, using reweighted flow matching \eqref{eq:weighted_matching} $w(\sbb, \ab)=\tilde g_f\left((Q(\sbb, \ab) - \nu(\sbb))/\lambda\right)$ might be negative. This result justifies the use of negative weights and indicates that existing negative-aware reweighting methods fit precisely within our measure matching framework.

\begin{myexample}[\textbf{WD1}~\cite{tang2025wd1} as an example of negative reweighting] \textbf{\emph{WD1}}~\cite{tang2025wd1} also incorporates negative reweighting to enhance the reasoning capability of diffusion language models. Their weighting function is defined as
$$
    \tilde g_f(\hat A(\sbb,\ab_i)) = \operatorname{softmax}(\hat A(\sbb,\ab_i)) - \operatorname{softmax}(-\hat A(\sbb,
\ab_i))\label{eq.reweighting_wd1}
$$
where $\operatorname{softmax}(\hat{A}(\sbb,\ab_i)) = \frac{\exp (\hat A(\sbb, \ab_i))}{\sum_i\exp(\hat A(\sbb, \ab_i))}$ and $\hat{A}(\sbb,\ab_i) = \frac{R(\sbb, \ab_i) - \operatorname{mean}(R(\sbb,\ab_{1:G}))}{\operatorname{std}(R(\sbb,\ab_{1:G}))}$ is the group-relative advantage across a total number of $G$ responses to the same query $\sbb$. 
Note that, $g(x)$ here is exactly monotonic increasing with respect to the advantage function.
\end{myexample}

Finally, we show that even without non-negativity constraints enforced, the target policy still improves over the current policy $\pi_{\rm old}$ as long as it is normalized. 

\begin{theorem}[Improvement of the signed target measure, proof in Appendix \ref{sec.apdx.policy_improvement}]
\label{thm:policy_improvement}
Assume $\tilde g_f$ is strictly increasing, the signed target policy measure $\tilde \pi^*$  has policy improvement over $\pi_{\rm old}$, \ie, $\EE_{\tilde \pi^*}[Q(\sbb, \ab)]\geq \mathbb{E}_{\pi_{\rm old}}[Q(\sbb, \ab)]$ for $\forall s\in\Scal$. 
\end{theorem}


\paragraph{The repelling effect caused by negative reweighting.}
In practice, negative reweighting allows the model to leverage low-value samples by explicitly pushing the generation trajectory away from them. Recall the optimal diffusion parameterization in~\eqref{eq.closed_form_solution_reweighting}. When $\mathbb{E}_{a_0 \sim p_{0|t}}[w(s,a_0)] > 0$, the numerator forms a weighted average of conditional velocity fields. For samples with negative weights, their corresponding conditional velocities are effectively reversed in direction. As a result, low-value samples contribute a repulsive velocity component, steering the learned generation path away from suboptimal regions rather than merely assigning them negligible weight. Cases where $\mathbb{E}_{a_0 \sim p_{0|t}}[w(s,a_0)] \leq 0$ can induce an analogous repelling effect, which we discuss in~\Cref{sec.geometric_interp}.

\section{Practical Algorithms}
\label{sec.algorithm}
\begin{algorithm}
\caption{\textbf{Ge}neralized \textbf{M}easure Matching for \textbf{P}olicy \textbf{O}ptimization (\algbb)}\label{alg.main}
\begin{algorithmic}[1]
    \REQUIRE $Q$-function, current policy $\pi_{\rm old}$, regularization strength $\lambda$, number of samples $N$
    \STATE sample $N$ actions $a_i\sim \pi(\cdot|\sbb),i=1,\dots N$ for each $\sbb$
    \STATE Evaluate $Q_i\doteq Q(\sbb,\ab_i),i=1,2,\dots N$
    \STATE Compute the empirical normalization factor $\nu(\sbb)$
    \STATE Compute reweighting weights $w(\sbb,\ab_i)$ with $\nu, Q,\lambda$
    \STATE Train policies by reweighted flow matching~\eqref{eq:weighted_matching}
\end{algorithmic}
\end{algorithm}

The pseudo-code of the \algbb framework is presented in Algorithm \ref{alg.main}. By instantiating the reweighted flow matching with different weighting functions, we study several variants of \algbb in this paper. Specifically, we consider:
1) \textbf{\algbb-Exp}, which uses exponential weighting as DPMD; 2) \textbf{\algbb-Linear} and \textbf{\algbb-Square}, which uses power functions $w(\sbb, \ab)=\left[(Q(\sbb, \ab)-\nu(\sbb)/\lambda\right]_+$ and $w(\sbb, \ab)=\left[(Q(\sbb, \ab)-\nu(\sbb))/\lambda\right]_+^2$, respectively. For variants that allow negative weighting, as discussed in \Cref{sec.geometric_interp}, negative weights must be carefully controlled to avoid destabilizing learning. To this end, we apply a simple truncation:
\begin{equation}\label{eq.negative_truncation}
w(\sbb, \ab) = \max\left(\frac{Q(\sbb, \ab)-\nu(\sbb)}{\lambda},\ell\right),
\end{equation}
where $\ell<0$ shapes the magnitude of negative weights. We denote this variant as \textbf{\algbb-Neg($\ell$)}. 

\textbf{Normalization constraints.}
It is not clear how to compute the normalization term $\nu(\sbb)$ for general reweighting functions. We propose a population-based solution by estimating the normalizer with a group of samples. For exponential weighting $g(\cdot)=\exp(\cdot)$, this translates to computing the softmax as the final weight; while for power function weighting, we can compute the normalizers by first determining the number of active samples (\ie, samples that are not clipped), and then compute the exact value such that the weights are normalized. Details can be found in Appendix~\ref{app:relu_derivations}.

\section{Experiments}
\label{sec.experiments}
In the experiments, we mainly answer the following questions:
\begin{enumerate}[nosep]
    \item[\qone] How do different weighting schemes affect the overall performance?
    \item[\qtwo] Do the negative weights help improve performance?
    \item[\qthree] How do we adapt reweighting functions according to the reward landscape?
\end{enumerate}

\subsection{Exploration-Exploitation Trade-off in Bandit}
We first study \qtwo~and \qthree~using a one-dimensional bandit problem, where we optimize a stateless policy $\pi(\cdot)$ to maximize an unknown reward function $R(\cdot)$. We evaluate performance using the simple regret metric, given by $\max_{a\in \mathcal{A}} R(a) - \mathbb E_{a\sim \pi}{[R(a)]}$.

\textbf{\qtwo~The effect of negative weights on exploration.} First, we select a reward function with two optima, shown in~\Cref{fig:bandit_negative} (Left), and initialize the policy to be aligned with the sub-optimum, so the policy is very likely to get stuck. We compare four types of reweighting methods: \textbf{Linear}, \textbf{Square}, \textbf{Exp}, and \textbf{Neg}. The first three are non-negative reweighting with different weighting functions in~\eqref{eq:weighted_matching}, while \textbf{Neg} uses negative reweighting in~\eqref{eq.negative_truncation} with default lower bound $\ell=-0.3$.

The regret curves are shown in~\Cref{fig:bandit_negative} (right), which show that, \textbf{Linear}, \textbf{Square}, \textbf{Exp} cannot achieve the global optimum for every trial. In particular, \textbf{Linear} reweighting always gets trapped in a local optimum. After adding negative reweighting, as shown in~\Cref{fig:bandit_negative} (middle), \textbf{Neg} can escape the sub-optimum, showing that negative reweighting can effectively help exploration and escape the local optimum. We also compare different negative lower bounds $\ell$ in~\Cref{fig:bandit_negative} (right). As $\ell$ decreases, the regret first decreases as the repelling effect kicks in, then regret increases again because large negative reweighting makes the algorithm unstable. The sweet spots are $-0.3$ or $-0.5$.  
%
\begin{figure}[h]
    \centering
    \vspace{-5pt}
    \includegraphics[height=0.17\linewidth]{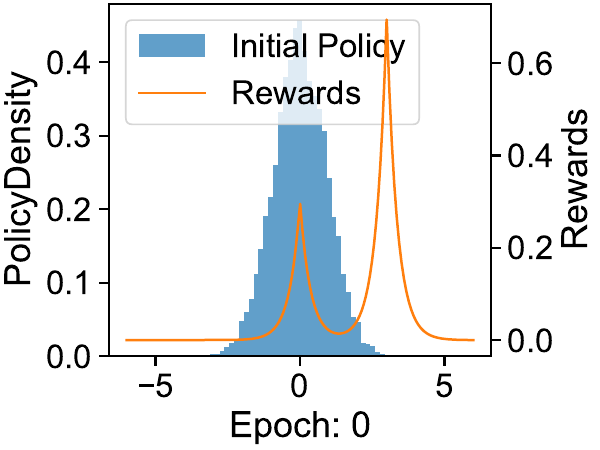}
    \includegraphics[height=0.17\linewidth]{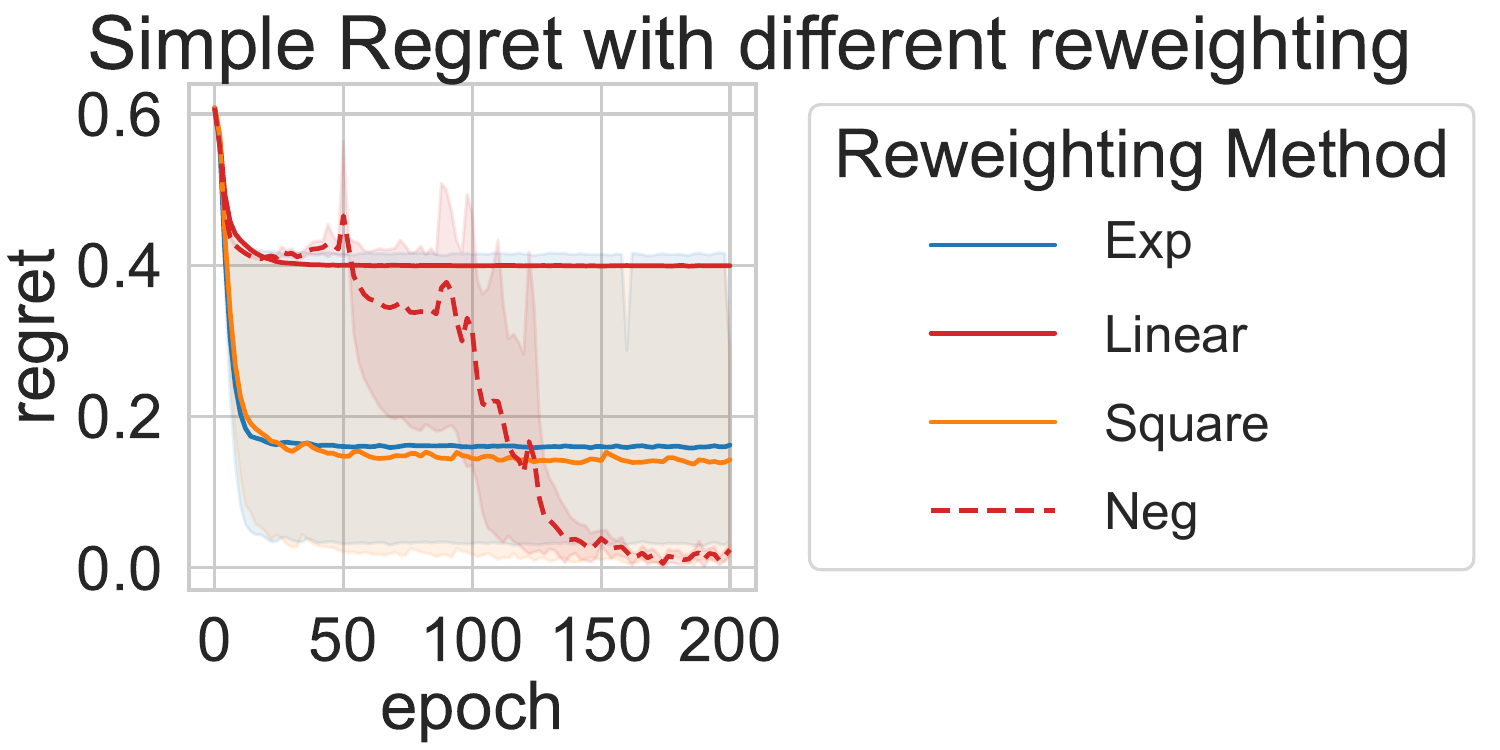}
    \includegraphics[height=0.17\linewidth]{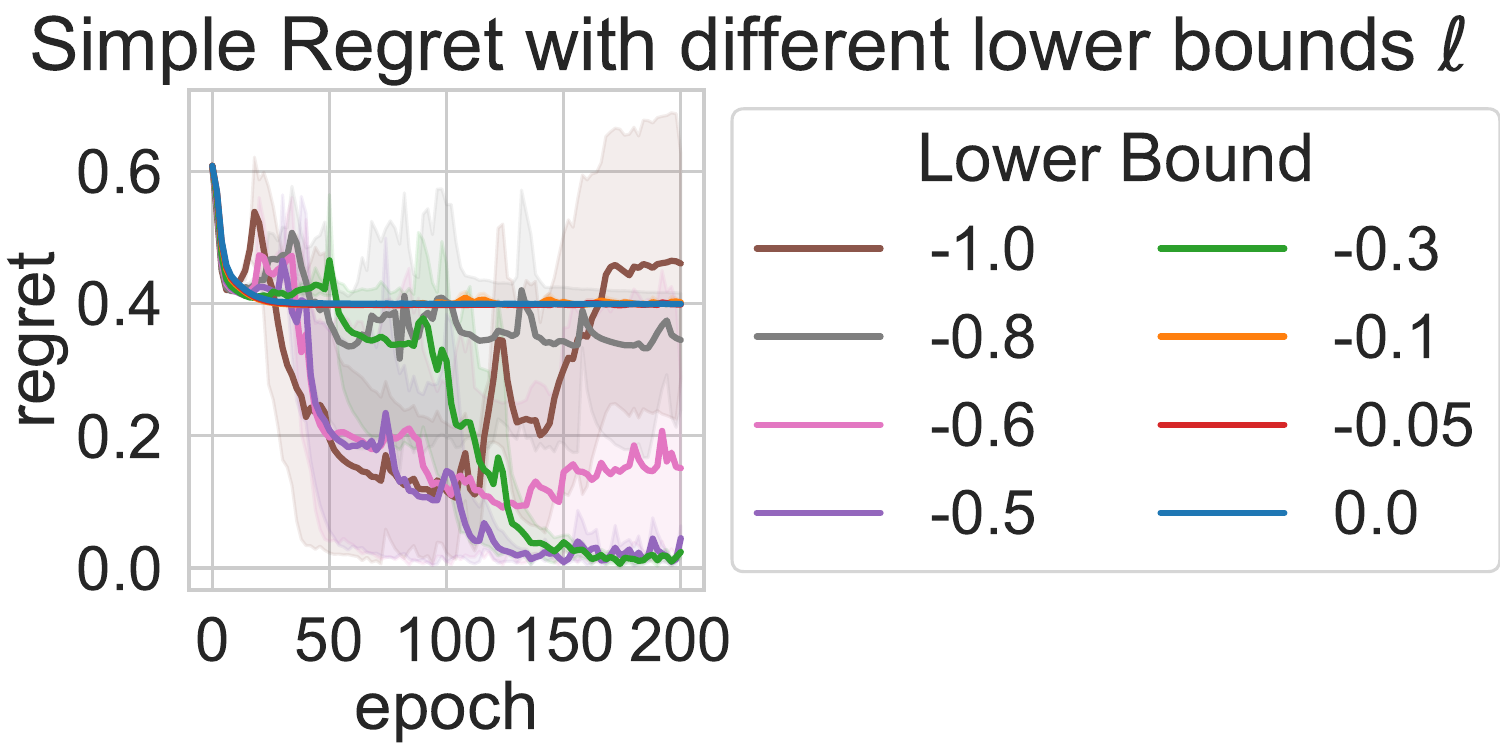}
    \vspace{-5pt}
    \caption{Bandit problem to answer \qtwo. \textbf{Left:} Reward functions and initial policy. \textbf{Middle:} Regret curves comparing different reweighting with and without negative reweighting functions over 200 epochs. \textbf{Right:} Regret curves with different lower bound $\ell$.} 
    \label{fig:bandit_negative}
\end{figure}
%
\begin{figure}[h]
    \centering
    \vspace{-3mm}    \includegraphics[height=0.18\linewidth]{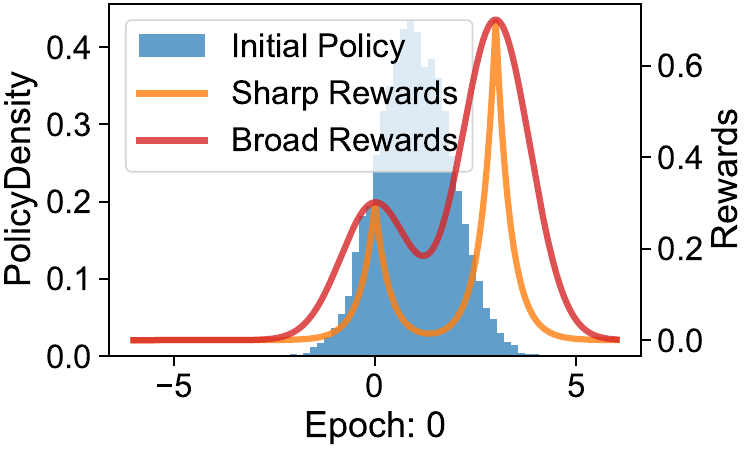}    \includegraphics[height=0.18\linewidth]{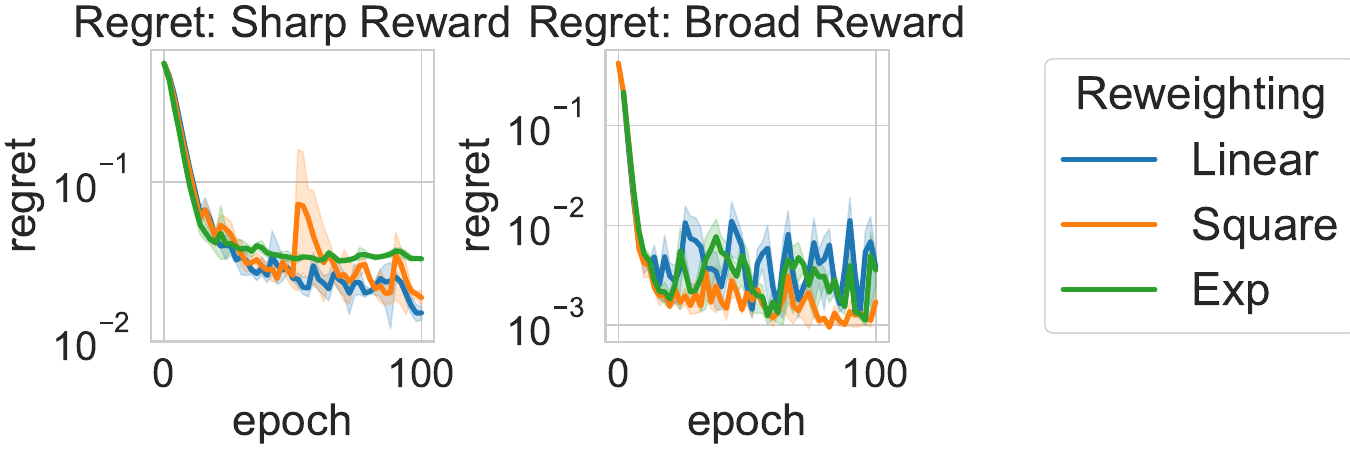}
    \vspace{-2mm}
    \caption{Bandit problem exploring \textcolor{purple}{\textbf{Q3}}. \textbf{Left:} Two reward functions and initial policy. \textbf{Right:} Regret curves comparing different reweighting schemes over 100 epochs. \textbf{Linear} excels in sharp rewards and \textbf{Square} excels at broad rewards, both outperform \textbf{Exp}.}
    \label{fig:bandit sample}
\end{figure}
\begin{table*}[ht]
\centering
\caption{Performance of 1M steps on Gym-MuJoCo environments. We report the average returns of the last 10 checkpoints $\pm$ standard error over 5 random seeds, and normalized average score normalized by SAC score. \textbf{Bold} numbers indicate best mean returns and \underline{underlined} numbers indicate mean returns within standard error of the best runs.}
\label{tab:performance}
\small
\setlength{\tabcolsep}{2pt}

\resizebox{\textwidth}{!}{
\begin{tabular}{l l S[table-format=5.0(4)] S[table-format=4.0(3)] S[table-format=4.0(4)] S[table-format=4.0(4)] S[table-format=4.0(4)] S[table-format=3.0(3)] S[table-format=3.1]}
\toprule
\multicolumn{2}{l}{\textbf{Method}} 
& {\textsc{HalfCheetah}} 
& {\textsc{Humanoid}} 
& {\textsc{Ant}} 
& {\textsc{Walker2d}} 
& {\textsc{Hopper}} 
& {\textsc{Swimmer}} 
& {\textbf{Norm. Avg.(\%)}} \\
\midrule
\multicolumn{2}{l}{\textbf{Classic Model-Free RL}} \\
& TD3  
& 9820(1543) 
& 5263(47)
& 4400(1196) 
& 3732(388) 
& \multicolumn{1}{c}{\boldcell{3387}{302}} 
& 88(25)
& 107.4 \\
& SAC  
& {\underlinecell{10938}{\ \ \ \ 68}} 
& 5131(97) 
& {\underlinecell{4945}{\ \ 186}} 
& 4463(151) 
& 2509(276) 
& 61(5)
& 100.0 \\
\midrule
\multicolumn{2}{l}{\textbf{Diffusion RL Baselines}} \\
& DIPO 
& 10136(2525) 
& 5184(132)
& 977(20)    
& 3809(1112)
& 1191(770)
& 46(3)
& 70.3 \\
& QSM 
& 10444(250) 
& 5103(73) 
& 3614(360) 
& 3635(621) 
& {\underlinecell{3200}{\ \ 255}} 
& 86(7)
& 103.0 \\
& DACER 
& 7669(1539) 
& {\boldcell{5332}{\ 64}} 
& {\boldcell{5141}{\ 356}} 
& 4252(176) 
& 2151(495) 
& 79(28)
& 98.1 \\
& QVPO
& 8655(209) 
& 5131(69) 
& 3852(67) 
& 3577(545) 
& {\underlinecell{3312}{\ \ \ \ 76}} 
& {\boldcell{110}{\ 5}} 
& 108.3 \\
\midrule
\multicolumn{2}{l}{\textbf{Ours}} \\
\rowcolor{sippo_std}
\rowcolor{sippo_std}
& \textbf{\algbb-Exp} 
& {\boldcell{10989}{\ 329}} 
& {\underlinecell{5293}{\ 55}} 
& 4719(58)
& {\boldcell{4695}{\ 197}} 
& {\underlinecell{3350}{\ \ \ \ 58}} 
& 99(13)
& 116.7 \\

\rowcolor{sippo_std}
& \textbf{\algbb-Square} 
& {\underlinecell{10679}{\ \ 500}} 
& 5179(49) 
& {\underlinecell{4942}{\ \ \ \ 75}} 
& {\underlinecell{4622}{\ \ \ \ 83}} 
& {\underlinecell{3325}{\ \ 113}} 
& 102(20)
& \textbf{117.0} \\

\rowcolor{sippo_std}
& \textbf{\algbb-Linear} 
& 9812(506) 
& 5265(74) 
& {\underlinecell{4819}{\ \ \ \ 73}} 
& {\underlinecell{4683}{\ \ \ \ 96}} 
& 3062(97) 
& 99(11)
& 113.2 \\
\bottomrule
\end{tabular}
}
\end{table*}

\begin{figure}[h]
    \centering
    \includegraphics[width=1.0\linewidth]{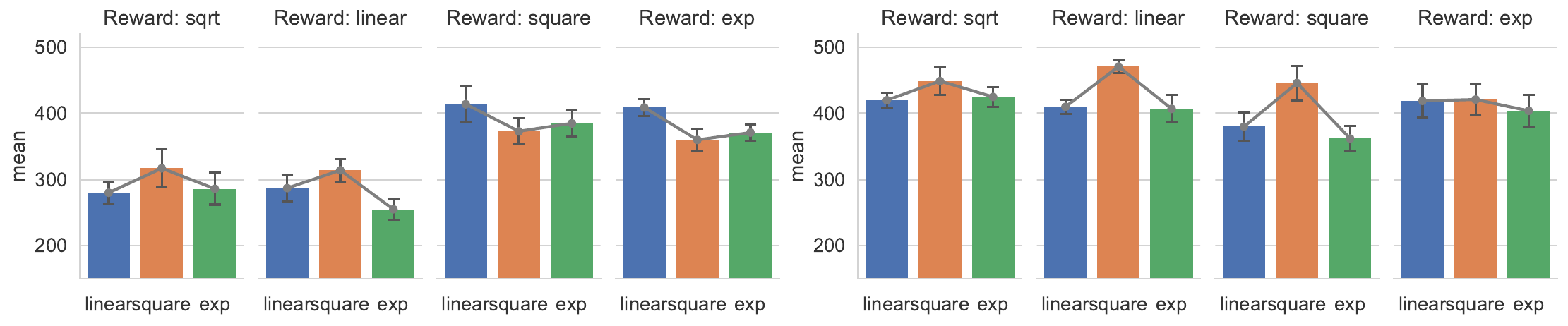}
    \vspace{-10pt}
    \caption{Performance with different reweighting and reward functions on the Cheetah Run (left) and Walker-Run (right) task. Boxes and error bars show mean returns and standard error over 5 seeds at the last checkpoint.}
    \label{fig:reward_landscape}
\end{figure}

\textbf{Reweighting and reward landscape.} Moreover, to answer \qthree, we compare two different reward functions, one with two broad optima and the other with two sharp optima, illustrated in \Cref{fig:bandit sample} (left). We show three non-negative reweighting functions: \textbf{Linear}, \textbf{Square}, and \textbf{Exp}. The resulting expected simple regret curves are presented in the \Cref{fig:bandit sample} (right). \textbf{Linear} excels in sharp rewards and \textbf{Square} excels at broad rewards, both outperform \textbf{Exp}. 

In summary, the choice of regularization term will affect the algorithm's performance, as it governs the exploration-exploitation trade-off relative to the curvature of the reward function.


\subsection{Locomotion Tasks}
To answer all questions on standard RL benchmarks, we implemented our algorithm with the JAX package and evaluated the performance on 6 OpenAI Gym MuJoCo v5 tasks. All environments are trained with a total of 1 million environment interactions. 
\textbf{Baselines} include two families of model-free RL algorithms. The first family is a collection of recent diffusion-policy online RL algorithms, including QSM~\cite{psenka2023learning}, QVPO~\cite{ding2024diffusion}, DACER~\cite{wang2024diffusion}, and DIPO~\cite{yang2023policy}. The second family is classic model-free online RL baselines including TD3~\cite{fujimoto2018addressing} and SAC~\cite{haarnoja2018soft}. A more detailed explanation to the baselines can be found in Appendix \ref{sec:apdx_baselines}.

\begin{wrapfigure}{r}{0.53\linewidth}
    \centering
    \vspace{-8pt}
\includegraphics[width=\linewidth]{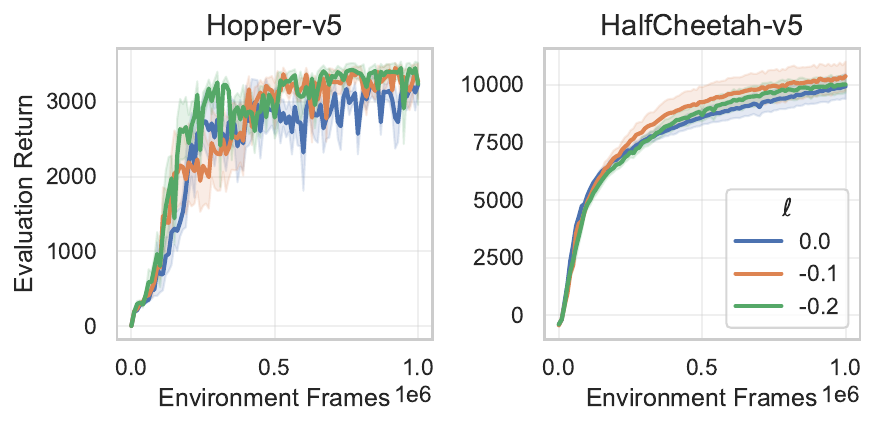}
\vspace{-15pt}
\caption{Training curves of \textbf{\algbb-Neg($\ell$)} with different truncation lower bound $\ell$. Setting $\ell=0$ corresponds to \textbf{\algbb-Linear}.}\label{fig.neg_ablation}
\end{wrapfigure}
\qone~The results are summarized in~\Cref{tab:performance}. All \textbf{\algbb} variants, \textbf{Linear}, \textbf{Square}, and \textbf{Exp}, exhibit consistently strong performance and outperform other diffusion-based RL baselines across most environments. Especially, \textbf{\algbb-Square} achieves the highest normalized average score, highlighting the benefit from flexible general reweighting function in our framework. 

\qtwo~Meanwhile, as shown in~\Cref{fig.neg_ablation}, we found that incorporating negative reweighting with $\ell=-0.1$ or $-0.2$ further enhances performance on top of the \text{Linear} variant, leading to marginal performance gains on the HalfCheetah-v5 and Hopper-v5 tasks.



\begin{wrapfigure}{r}{0.18\textwidth}
    \centering
    \includegraphics[width=\linewidth]{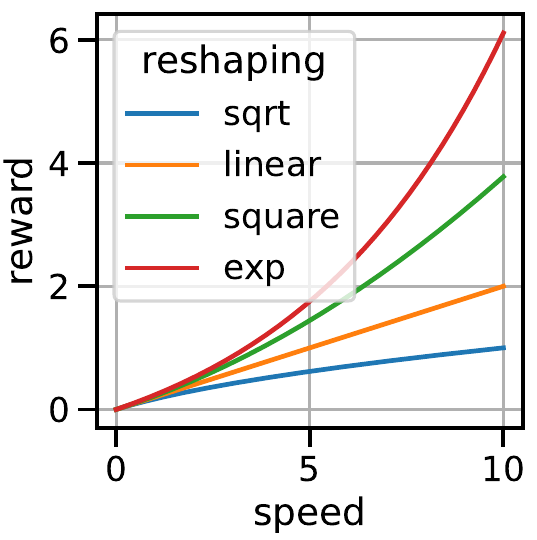}
    \vspace{-20pt}
    \caption{Reward functions. }
    \label{fig:reward_shaping}
\end{wrapfigure}
\qthree~To see the effect of reward functions under different landscapes in locomotion tasks, we handcrafted different reward functions to change the landscape in two DeepMind Control Suite tasks~\cite{tassa2018deepmind}, Cheetah-Run and Walker-Run, shown in~\Cref{fig:reward_shaping}. The task is to run as fast as possible, where rewards are monotonically increasing functions of robot speed. Four different rewards are ordered from flat to greedy as \texttt{sqrt}, \texttt{linear}, \texttt{square}, \texttt{exp}. The detailed definition of these reward functions is deferred to Appendix~\ref{sec.apdx.reward}.
The results are shown~\Cref{fig:reward_landscape}, which shows that \textbf{Square} reweighting excels at {flat} reward functions such as \texttt{sqrt} and \texttt{linear}, while \textbf{Linear} reweighting has better performance on {steep} reward functions, such as \texttt{exp}. The finding is consistent with our observations in the bandit problem. Therefore, we can choose \textbf{Linear} reweighting for steep reward landscapes and \textbf{Square} reweighting for relatively flat reward landscapes. Note that this is only qualitative intuition rather than saying one reweighting function must be better than another.

\vspace{-2pt}
\subsection{DNA Sequence Generation}
Finally, we evaluate our method on a real-world problem, learning to generate DNA sequences with a discrete diffusion model\footnote{Despite different definitions, discrete diffusion models share most of the properties with continuous diffusion, such as the optimal solution in~\Cref{eq:unweighted_closed_form_velocity}. Please refer to \citet{lou2023discrete,shi2024simplified} for details.}. The task is to fine-tune a DNA generation model, learned from a large public enhancer dataset containing roughly 700,000 DNA sequences~\citep{gosai2023machine}, to optimize the gene expression activity~\cite{avsec2021effective}. We leverage pretrained generation and reward models from~\cite{wang2024fine}. We tested multiple variants of our algorithm, including \textbf{Linear}, \textbf{Square}, and \textbf{Neg}. 

\textbf{Baselines.} We compare against controlled generation approaches including conditional guidance \citep[CG,][]{nisonoff2024unlocking}, SMC and ~\citep[TDS,][]{wu2023practical}, and classifier-free guidance \citep[CFG,][]{ho2022classifier}, as well as a strong RL-based baseline, DRAKES~\citep{wang2024fine}, which leveraged reparameterized policy gradient using Gumbel-softmax, and RL-D$^2$~\cite{ma2025reinforcement}, which uses weighted regression with KL divergence regularization (the same as \textbf{\algbb-Exp}).
\begin{table}[t]
\centering
\caption{\textbf{Main Results on Pred-Activity.} Comparison of \algbb variants against baselines. Standard \algbb variants are highlighted in orange, while variants with negative sample awareness are highlighted in blue and named as Neg($\ell$). The inclusion of negative samples consistently provides the highest gains up to $14.3\%$.}
\label{tab:simpo_colored}
\small
\setlength{\tabcolsep}{4pt}

\begin{tabular}{@{}c@{\hspace{8pt}}|@{\hspace{8pt}}c@{}}

\begin{tabular}[c]{@{}l S[table-format=1.2(2)] r@{}}
\toprule
\multicolumn{3}{c}{\textbf{Baselines}} \\
\cmidrule(lr){1-3}
\textbf{Method} & {\textbf{Pred.} $\uparrow$} & \textbf{Improv.} \\
\midrule
Pretrained & 0.17(04) & -- \\
CG & 3.30(01) & -- \\
TDS & 4.64(21) & -- \\
CFG & 5.04(06) & -- \\
DRAKES & 6.44(04) & -1.2\% \\
DRAKES-KL & 5.61(07) & -13.9\% \\
WD1 & 3.78(12) & -42.0\% \\
\textbf{RL-D$^2$ (best baseline)}
& \bfseries 6.52(03)
& -- \\
\bottomrule
\end{tabular}

&

\begin{tabular}[c]{@{}l S[table-format=1.2(2)] r@{}}
\toprule
\multicolumn{3}{c}{\textbf{\algbb Variants}} \\
\cmidrule(lr){1-3}
\textbf{Method} & {\textbf{Pred.} $\uparrow$} & \textbf{Improv.} \\
\midrule
\cellcolor{sippo_std}{Linear}
& \cellcolor{sippo_std} 6.54(01)
& \cellcolor{sippo_std} +0.3\% \\

\cellcolor{sippo_std}\textbf{Square}
& \cellcolor{sippo_std}\bfseries 6.77(03)
& \cellcolor{sippo_std}\textbf{+3.8\%} \\

\cellcolor{sippo_neg}{Neg(-0.05)}
& \cellcolor{sippo_neg} 6.75(05)
& \cellcolor{sippo_neg}{+3.5\%} \\

\cellcolor{sippo_neg}{Neg(-0.1)}
& \cellcolor{sippo_neg} 6.95(07)
& \cellcolor{sippo_neg}{+6.6\%} \\

\cellcolor{sippo_neg}\textbf{Neg(-0.3)}
& \cellcolor{sippo_neg}\bfseries 7.45(03)
& \cellcolor{sippo_neg}\textbf{+14.3\%} \\

\cellcolor{sippo_neg}{Neg(-0.5)}
& \cellcolor{sippo_neg} 5.62(113)
& \cellcolor{sippo_neg}{-13.8\%} \\
\bottomrule
\end{tabular}

\end{tabular}
\vspace{-10pt}
\end{table}

\textbf{Results.} Our experimental results, summarized in \Cref{tab:simpo_colored}, demonstrate that \algbb significantly outperforms all considered baselines in terms of the gene expression levels. While the best-performing baseline, RL-D$^2$, achieves a median score of $6.52$, our standard linear and square \algbb variants already provide competitive or superior performance. However, the most substantial gains are observed when incorporating negative sample awareness. Specifically, \textbf{\algbb-Neg(-0.3)} variant achieves the highest overall performance of $7.45 \pm 0.03$, representing a $+14.3\%$ improvement over the best baseline. Moreover, the \textbf{\algbb-Neg} variants show that the performance first increases and then decreases as $\ell$ decreases from $0$ to $-0.5$.

\vspace{-2pt}
\subsection{Sensitivity Analysis}

\begin{figure}[h]
    \centering
    \vspace{-10pt}\includegraphics[width=0.48\linewidth]{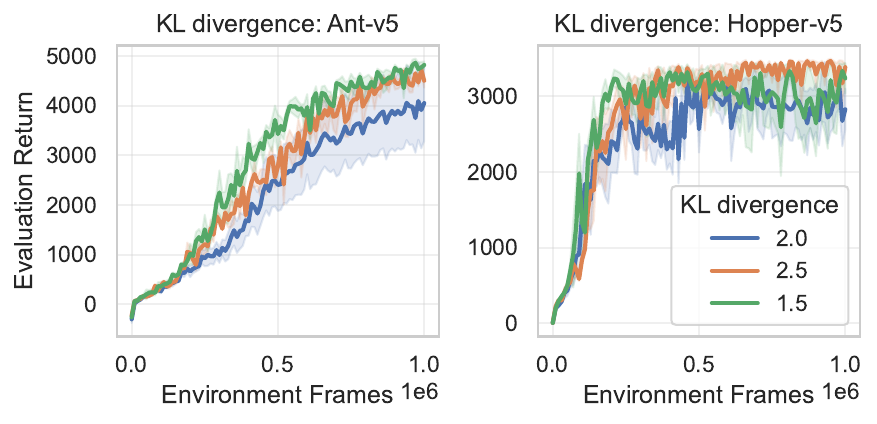}
    \includegraphics[width=0.48\linewidth]{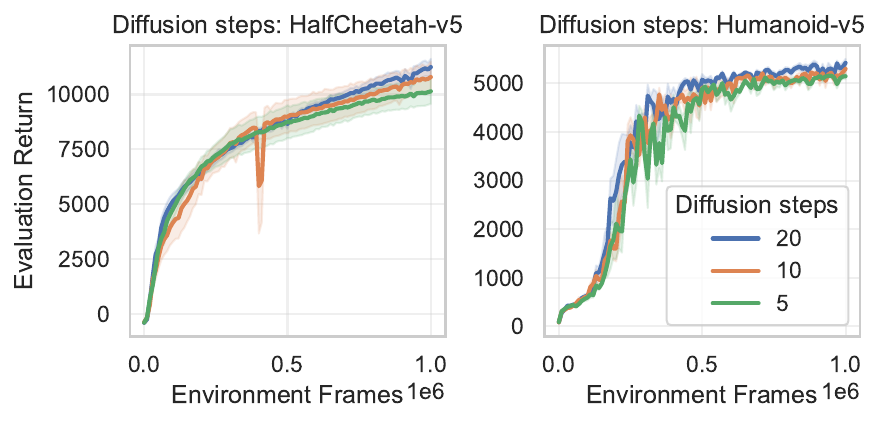}
    \caption{Ablation studies on KL tolerance and diffusion steps.}
    \vspace{-6pt}
    \label{fig:kl_ablation}
\end{figure}

\textbf{Ablation on the KL tolerance.} \algbb employs automatic temperature tuning by constraining the KL divergence between the weighting distribution $w(\cdot)$ and the uniform distribution (Appendix~\ref{apdx.temp_tuning}). A higher KL tolerance permits more aggressive temperature adjustments, leading to a greedier reweighting. Our evaluation (\Cref{fig:kl_ablation}) indicates that the optimal tolerance is task-dependent, necessitating hyperparameter tuning across different domains.

\textbf{Ablation on diffusion steps.} Results in \Cref{fig:kl_ablation} indicate that performance is largely insensitive to the number of diffusion steps. While we observe a marginal improvement in accuracy as the step count increases, the overall gains remain minimal, suggesting that \algbb is robust to this hyperparameter.




\vspace{-1.5mm}
\section{Conclusions}
\vspace{-1.5mm}
\label{sec.conclusion}
In this work, we established a cohesive framework which uses Generalized Measure matching for diffusion Policy Optimization~(\algbb). By decoupling algorithm design into target measure construction and measure matching, we clarify the broader design space of reweighted flow matching. We demonstrate that by adopting different $f$-divergences, our framework recovers several existing algorithms, such as QVPO and DPMD, as its special cases. Furthermore, by extending the target to signed measure families, our framework provides a justification for negative weights, thereby resolving a critical exploration dilemma inherent in previous reweighted flow matching methods. One limitation of this negative reweighting is that it might cause an unbounded loss function, while our results showed that proper $\ell$ could avoid instability. We hope this paper will inspire more efficient and effective post-training algorithms for diffusion models and flow models.



\newpage
\bibliography{ref,sdac}
\bibliographystyle{plainnat}


\newpage
\appendix

\section{Related works}
\label{sec.related_works}
\textbf{RL for diffusion and flow models.} RL for diffusion or flow models can be classified into two categories: \textbf{a)} Only working with the reverse process, either by computing \emph{policy gradient} via accumulating the log probabilities over the diffusion timesteps~\cite{black2023training,fan2023dpok,ren2024diffusion,liu2025flow,zhang2025reinflow}, or directly backpropagating gradients through the reverse process by \emph{reparameterization}~\cite{clark2023directly,wang2024diffusion,prabhudesai2024video,celik2025dime,ding2025genpo,zhang2025sac,gaobehavior}. Such methods closely resemble standard RL algorithms, such as Proximal Policy Optimization (PPO)~\cite{schulman2017proximal} and Soft Actor Critic (SAC)~\cite{haarnoja2018soft}, respectively. However, their computational cost typically scales linearly with the number of reverse diffusion steps \citep{wang2022diffusion,gaobehavior}. Moreover, they require a training pipeline that is substantially different from standard score or flow matching, as we need to generate and differentiate through the denoising trajectory. \textbf{b)} Another way is preserving the forward process, and steering the score or flow prediction through \emph{weighted behavior cloning} \citep{ding2024diffusion,ma2025efficient,zhang2025energy,tang2025wd1,dong2025maximum,li2026reverse} or \emph{moving towards calibrated targets} \citep{lu2023contrastive,psenka2023learning,frans2025diffusion,akhound2024iterated}. In this way, the optimization is well-formulated as distribution matching. In this paper, we dive into the weighted flow matching formula and explore more reweighting options.


\textbf{Regularizations in policy extraction.} KL divergence is widely adopted to regularize the policy improvement~\cite{peng2019advantage, vieillard2020leverage,nair2020awac}. However, it is also possible to adopt other types of divergence metrics. Several works also consider reverse KL regularization~\cite{xu2025uni,ma2025reinforcement}, leading to update rules similar to policy gradient theorems. \citet{xu2023offline,xu2025uni} further study implicit value regularization, which derives closed-form target policies and update rules for general $f$-divergence regularized problems. In this paper, we deepen this understanding by extending it to both diffusion policy and signed target measures, thus allowing more flexible and effective objective design. Finally, another branch of research about policy mirror descent is motivated by the Bregman divergence~\cite{tomar2021mirrordescentpolicyoptimization,shani2020adaptive}. Notably, the KL divergence occupies a special position, as it is both a Bregman divergence and an $f$-divergence.

\textbf{Leveraging negative samples in RL for diffusion models.} Negative weights are typically underutilized in weighted behavior cloning and best-of-N distillation. To better exploit such samples in diffusion models, \citet{frans2025diffusion,intelligence2025pi,zheng2025diffusionnft} used reward or advantage functions to classify data into preference pairs, and reinforced the diffusion model with classifier-free guidance. \citet{tang2025wd1} leveraged a weighted behavior cloning with negative weights for diffusion language model reasoning. However, the negative reweighting function is an ad-hoc design, while we provide a generalized and unified viewpoint with theoretical justification.
\section{Extension of $f$-divergence to Signed Measures.}
\label{sec.f_div_signed_measure}
Let \(Q\) be a finite nonnegative measure (e.g., a probability measure) and let \(P\) be a finite signed measure. Assume the total variation of \(P\) is absolutely continuous with respect to \(Q\), i.e., \(|P|\ll Q\), so that the Radon--Nikodym derivative \(r := dP/dQ\) exists \(Q\)-a.e.\ and may take negative values. If \(f\) is extended to a measurable function on \(\mathbb{R}\) (or at least on the essential range of \(r\)) such that \(f(r)\in L^{1}(Q)\), we define
\[
D_f(P\|Q)\;:=\;\int f\!\left(\frac{dP}{dQ}\right)\,dQ.
\]
Moreover, if \(Q\) is a probability measure, \(P(\Omega)=Q(\Omega)=1\), and the extension of \(f\) is convex on \(\mathbb{R}\) with \(f(1)=0\), then Jensen's inequality yields \(D_f(P\|Q)\ge 0\) despite \(r\) potentially being negative; under strict convexity, the usual identity-of-indiscernibles conclusion \(D_f(P\|Q)=0\Rightarrow P=Q\) follows (up to standard regularity/equality conditions).

\section{Derivations of the Regularized Policy Optimization}

\subsection{Derivation of the $f$-Divergence Regularized Policy Optimization}\label{sec.apdx.derivation}
\label{sec.apdx.derivation.optimal_solution}

We omit the state $\sbb$ in the following derivation for brevity. Let $f: \mathbb{R}^+\to\mathbb{R}$ be strictly convex and differentiable and $f(1)=0$. We want to find a policy $\pi\in\Pi$ that maximizes the value $Q(\ab)$ while controlling the $f$-divergence from the behavior policy $\pi_{\rm old}$, which can be formulated as the following constrained optimization problem:
\begin{equation}\label{problem:f-div}
\begin{aligned}
&\quad\max_{\pi\in\Pi}\  \sum_{\ab\in\Acal} \pi(\ab) Q(\ab) \\
&\text{s.t.}\left\{\begin{aligned}
    D_{f}(\pi(\cdot)\|\pi_{\rm old}(\cdot))&\leq \epsilon \\
    \sum_{\ab\in\Acal} \pi(\ab)&=1\\
    \forall \ab\in\Acal, \ \ \pi(\ab)&\geq 0
\end{aligned}\right.
\end{aligned}
\end{equation}

We introduce Lagrange multipliers $\lambda$, $\nu$, and $\eta(\ab)\geq 0$ for the constraints, respectively, leading to the following Lagrange function $\Lcal$,
\begin{equation}\label{eq:f-div_lag}
\begin{aligned}
    \Lcal(\pi, \lambda, \nu, \eta) &= \sum_{\ab\in\Acal} \pi(\ab) Q(\ab) - \lambda \left(\sum_{\ab\in\Acal}\pi_{\rm old}(\ab)f\rbr{\frac{\pi(\ab)}{\pi_{\rm old}(\ab)}}-\epsilon\right) - \nu(1-\sum_{\ab\in\Acal}\pi(\ab))+\sum_{\ab\in\Acal}\eta(\ab)\pi(\ab).
\end{aligned}
\end{equation}
We take the functional derivative with respect to $\pi(\ab)$ and set it to zero. Recall that the derivative of the term $\mu f(\pi / \mu)$ with respect to $\pi$ is $f^{\prime}(\pi / \mu)$, 
\begin{equation}
\begin{aligned}
    0=\frac{\partial \Lcal}{\partial \pi(\ab)} =Q(\ab)-\lambda f'(\frac{\pi(\ab)}{\pi_{\rm old}(\ab)})-\nu+\eta(\ab),
\end{aligned}
\end{equation}
Rearranging to isolate the divergence term,
\begin{equation}
f'\left(\frac{\pi(\ab)}{\pi_{\rm old}(\ab)}\right)=\frac{Q(\ab)-\nu+\eta(\ab)}{\lambda}.
\end{equation}

We define $g_f=(f')^{-1}$. Since the complimentary slackness requires $\eta(\ab)\pi(\ab)=0$, we consider the following cases:
\begin{itemize}
    \item $\pi(\ab)>0$, in this case $\eta(\ab)=0$, therefore we have 
    $$
    \begin{aligned}
    f'\left(\frac{\pi(\ab)}{\pi_{\rm old}(\ab)}\right)=\frac{Q(\ab)-\nu}{\lambda}&\Rightarrow \frac{\pi(\ab)}{\pi_{\rm old}(\ab)}=g_f\left(\frac{Q(\ab)-\nu}{\lambda}\right)\\
    \end{aligned}
    $$
    Besides, since $f$ is strongly convex, both $f'$ and $g_f$ are increasing, therefore 
    $$
    \frac{Q(\ab)-\nu}{\lambda}=f'(\frac{\pi(\ab)}{\pi_{\rm old}(\ab)})>f'(0).
    $$
    \item $\pi(\ab)=0$, in this case $\eta(\ab)>0$. From the stationary condition, we know that 
    $$
    Q(\ab)-\nu +\eta(\ab)=\lambda f'(0) \Rightarrow \frac{Q(\ab)-\nu}{\lambda}< f'(0).
    $$
\end{itemize}
Therefore, for both cases, we can unify the optimal solution as
\begin{equation}\label{eq:f-div_optimal_sol}
\begin{aligned}
    \pi(\ab)&= \pi_{\rm old}(\ab)\bar{g}_f\left(\frac{Q(\ab)-\nu}{\lambda}\right),
\end{aligned}
\end{equation}
where 
\begin{equation}
    \bar{g}_f(x)=\left\{
    \begin{aligned}
        &g_f(x)\quad &\text{if }x>f'(0)\\
        &0 & \text{if }x\leq f'(0)
    \end{aligned}
    \right..
\end{equation} 

\subsection{Derivation of the $f$-Divergence Regularized Policy Optimization on Signed Measures}
\label{sec.apdx.derivation.optimal_solution_signed}

Let $f(x): \mathbb{R}\to\mathbb{R}$ be strictly convex and differentiable and $f(1)=0$. Since $f$ is strictly convex, $f'$ is monotonically increasing, and therefore $g_f=(f')^{-1}$ is well defined. Consider the solution space $\widetilde{\Pi}=\{\pi:\sum_{\ab\sim\Acal}\pi(\ab)=1\}$, the optimization problem regularized the generalized $f$-divergence is defined as
\begin{equation}\label{problem:slack}
\begin{aligned}
&\quad\max_{\pi\in\widetilde{\Pi}}\  \sum_{\ab\in\Acal} \pi(\ab) Q(\ab) \\
&\text{s.t.}\left\{\begin{aligned}
    D_{{f}}(\pi(\cdot)\|\pi_{\rm old}(\cdot))&\leq \epsilon \\
    \sum_{\ab\in\Acal} \pi(\ab)&=1\\
\end{aligned}\right.
\end{aligned}
\end{equation}

We introduce Lagrange multipliers $\lambda$ and $\nu$ for the constraints, leading to the following Lagrange function $\Lcal$:
\begin{equation}
\begin{aligned}
    \Lcal(\pi, \lambda, \nu) &= \sum_{\ab\in\Acal} \pi(\ab) Q(\ab) - \lambda \left(\sum_{\ab\in\Acal}\pi_{\rm old}(\ab)f\rbr{\frac{\pi(\ab)}{\pi_{\rm old}(\ab)}}-\epsilon\right) - \nu(1-\sum_{\ab\in\Acal}\pi(\ab)).
\end{aligned}
\end{equation}
Taking its derivative w.r.t. $\pi(\ab)$ and setting it to zero:
\begin{equation}
\begin{aligned}
    0=\frac{\partial \Lcal}{\partial \pi(\ab)} =Q(\ab)-\lambda f'\rbr{\frac{\pi(\ab)}{\pi_{\rm old}(\ab)}}-\nu,
\end{aligned}
\end{equation}
i.e., 
$$
\begin{aligned}
    \pi(\ab)=\pi_{\rm old}(\ab)(f')^{-1}\left(\frac{Q(\ab)-\nu}{\lambda}\right)=\pi_{\rm old}(\ab)g_f\left(\frac{Q(\ab)-\nu}{\lambda}\right).
\end{aligned}
$$
Note that $\nu$ is the normalization term such that $\sum_{\ab\in\Acal}\mu(\ab)g\left((Q(\ab)-\nu)/\lambda\right)=1$.

Generally speaking, for any strictly increasing (and typically continuous) function $g$ defined on an interval $I\subseteq \mathbb{R}$, its inverse $g^{-1}$ exists on $g(I)$ and is strictly increasing. Define $f'_g(x)=g^{-1}(x)$ for $x\in g(I)$ and 
$$
f_g(x)=\int_{\ab_0}^xg^{-1}(t)\mathrm{d}t + C
$$
yields a differentiable strictly convex function on $g(I)$. Normalizing $f_g(1)=0$ underpins $C$. While the analytic formulation of $f_g$ may not be available for an arbitrary $g$, this observation allows us to reason directly in terms of the weighting function $g$ rather than $f_g$. 

\subsection{Policy Improvement Property of the Target Measure}\label{sec.apdx.policy_improvement}
\begin{theorem}[Theorem \ref{thm:policy_improvement} restated] Assume $g: \mathbb{R}\to\mathbb{R}$ is monotonic increasing and the target measure $\tilde{\pi}$ is defined as
$$
\tilde{\pi}(\ab|\sbb)=\pi(\ab|\sbb)g\left(\frac{Q(\sbb, \ab)-\nu}{\lambda}\right)\quad \forall \sbb\in\Scal,
$$
where $\lambda>0$ and $\nu$ is selected such that $\sum_{\ab\in\Acal}\tilde{\pi}(\ab|\sbb)=1$ holds. Then we have $$
\sum_{\ab\in\Acal} \tilde{\pi}(\ab|\sbb)Q(\sbb, \ab) \geq \sum_{\ab\in\Acal} \pi(\ab|\sbb)Q(\sbb, \ab)\quad \forall \sbb\in\Scal.
$$
\end{theorem}
\begin{proof}
    Denote $A(\sbb, \ab)=\frac{Q(\sbb, \ab)-\nu}{\lambda}$, and thus we have $\sum_{\ab\in\Acal}\tilde{\pi}(\ab|\sbb)=\mathbb{E}_{\pi}[g\left(A(\sbb, \ab)\right]=1$. It suffices to show 
    $$
    \sum_{\ab\in\Acal}\tilde{\pi}(\ab|\sbb)A(\sbb, \ab) \geq \sum_{\ab\in\Acal}\pi(\ab|\sbb)A(\sbb, \ab).
    $$
    Consider the difference
    $$
    \begin{aligned}
        &\sum_{\ab\in\Acal}\tilde{\pi}(\ab|\sbb)A(\sbb, \ab) - \sum_{\ab\in\Acal}\pi(\ab|\sbb)A(\sbb, \ab)\\
        &=\sum_{\ab\in\Acal}\pi(\ab|\sbb)A(\sbb, \ab)g\left(A(\sbb, \ab)\right) - \left(\sum_{\ab\in\Acal}\pi(\ab|\sbb)A(\sbb, \ab)\right)\left(\sum_{\ab\in\Acal}\pi(\ab|\sbb)g\left(A(\sbb, \ab)\right)\right)\\
        &=\EE_\pi[A(\sbb, \ab)g\left(A(\sbb, \ab)\right)]-\EE_\pi[A(\sbb, \ab)]\EE_\pi[g\left(A(\sbb, \ab)\right)]\\
        &=\text{Cov}_{\pi}[A(\sbb, \ab), g\left(A(\sbb, \ab)\right)].
    \end{aligned}
    $$
    Since $g$ is monotonic increasing, we know that $\text{Cov}_{\pi}[A(\sbb, \ab), g\left(A(\sbb, \ab)\right)]\geq 0$, and thus
    $$
    \sum_{\ab\in\Acal}\tilde{\pi}(\ab|\sbb)A(\sbb, \ab) \geq \sum_{\ab\in\Acal}\pi(\ab|\sbb)A(\sbb, \ab).
    $$
\end{proof}

\section{Derivation of Closed-Form Solution of Reweighted Conditional Flow Matching.}
\label{apdx.closed_form}
We temporarily omit conditioning on the state $\sbb$ for clarity.

\subsection{Closed-Form Solution of Conditional Flow Matching}
We first establish the optimal solution for standard conditional flow matching without a reweighting factor $w(\boldsymbol{x}_0)$.\begin{lemma}[Conditional flow matching fits marginal velocity~\cite{lipman2022flow}]The conditional flow matching (CFM) loss is defined as:
$$
\mathcal{L}_{\rm CFM}(\theta)=\mathbb{E}_{t \sim \mathcal{U}[0,1], \ \boldsymbol{x}_0 \sim p_{\text{data}}(\boldsymbol{x}_0), \ \boldsymbol{x}_t \sim q_{t|0}(\cdot|\boldsymbol{x}_0)} \left[ \| D_{\theta}(\boldsymbol{x}_t, t) - \boldsymbol{v}_{t|0}(\boldsymbol{x}_t|\boldsymbol{x}_0) \|^2 \right].
$$
Upon convergence, the $D_{\theta^*}(\boldsymbol{x_t}, t)$ fits the marginal velocity $\boldsymbol{v}_t(\boldsymbol{x}_t)$. 
\end{lemma}
\begin{proof}
Expanding the expectation into integral form and applying Bayes' theorem to reverse the joint distribution, $p_{\text{data}}(\boldsymbol{x}_0) q_{t|0}(\boldsymbol{x}_t|\boldsymbol{x}_0) = q_t(\boldsymbol{x}_t) q_{0|t}(\boldsymbol{x}_0|\boldsymbol{x}_t)$, we can rewrite the loss as:$$\mathcal{L}_{\rm CFM}(\theta)=\int_0^1 \int q_t(\boldsymbol{x}_t) \left[ \int q_{0|t}(\boldsymbol{x}_0|\boldsymbol{x}_t) \| D_\theta(\boldsymbol{x}_t, t) - \boldsymbol{v}_{t|0}(\boldsymbol{x}_t|\boldsymbol{x}_0) \|^2 \, d\boldsymbol{x}_0 \right] \, d\boldsymbol{x}_t \, dt$$Because $q_t(\boldsymbol{x}_t)$ and $dt$ are non-negative, minimizing the global integral is equivalent to minimizing the inner objective pointwise for any fixed $\boldsymbol{x}_t$ and $t$. Let $v=D_\theta(\boldsymbol{x}_t, t)$. The local objective becomes:$$J(v)=\int q_{0|t}(\boldsymbol{x}_0|\boldsymbol{x}_t) \| v - \boldsymbol{v}_{t|0}(\boldsymbol{x}_t|\boldsymbol{x}_0) \|^2 \, d\boldsymbol{x}_0$$Setting the gradient with respect to $v$ to zero yields:$$\nabla_v J(v)=2 \int q_{0|t}(\boldsymbol{x}_0|\boldsymbol{x}_t) (v - \boldsymbol{v}_{t|0}(\boldsymbol{x}_t|\boldsymbol{x}_0)) \, d\boldsymbol{x}_0 = 0$$Since $v$ is independent of $\boldsymbol{x}_0$ and the posterior integrates to 1 ($\int q_{0|t}(\boldsymbol{x}_0|\boldsymbol{x}_t) d\boldsymbol{x}_0 = 1$), solving for $v$ provides the optimal vector field:
$$
v^*=D_{\theta^*}(\boldsymbol{x}_t, t)=\mathbb{E}_{\boldsymbol{x}_0 \sim q_{0|t}(\boldsymbol{x}_0|\boldsymbol{x}_t)} [ \boldsymbol{v}_{t|0}(\boldsymbol{x}_t|\boldsymbol{x}_0) ]
$$
By definition, this matches the marginal velocity field $\boldsymbol{v}_t(\boldsymbol{x}_t) \doteq \mathbb{E}_{\boldsymbol{x}_0 \sim q_{0|t}(\boldsymbol{x}_0|\boldsymbol{x}_t)} [ \boldsymbol{v}_{t|0}(\boldsymbol{x}_t|\boldsymbol{x}_0) ]$, concluding the proof.
\end{proof}

\subsection{Closed-Form Solution of Reweighted Conditional Flow Matching}
We now present our primary result regarding reweighted conditional flow matching.\begin{theorem}[Optimal solution of reweighted conditional flow matching]When $\mathbb{E}_{\boldsymbol{x}_0 \sim q_{0|t}} [ w(\boldsymbol{x}_0) ] > 0$, the reweighted conditional flow matching loss
$$\mathcal{L}_{W}(\theta)=\mathbb{E}_{t \sim \mathcal{U}, \ \boldsymbol{x}_0 \sim p_{\text{data}}(\boldsymbol{x}_0), \ \boldsymbol{x}_t \sim q_{t|0}(\boldsymbol{x}_t|\boldsymbol{x}_0)} \left[ w(\boldsymbol{x}_0) \| D_\theta(\boldsymbol{x}_t, t) - \boldsymbol{v}_{t|0}(\boldsymbol{x}_t|\boldsymbol{x}_0) \|^2 \right]$$induces the optimal solution:
$$
D_{\theta^*}(\boldsymbol{x}_t;t)=\frac{\mathbb{E}_{\boldsymbol{x}_0 \sim q_{0|t}} [ w(\boldsymbol{x}_0) \boldsymbol{v}_{t|0}(\boldsymbol{x}_t|\boldsymbol{x}_0) ]}{\mathbb{E}_{\boldsymbol{x}_0 \sim q_{0|t}} [ w(\boldsymbol{x}_0) ]}
$$
\end{theorem}
\begin{proof}
Let $v=D_\theta(\boldsymbol{x}_t, t)$. Expanding the expectation and substituting the marginal-posterior decomposition $q_t(\boldsymbol{x}_t) q_{0|t}(\boldsymbol{x}_0|\boldsymbol{x}_t)$, we isolate the pointwise loss for a fixed $\boldsymbol{x}_t$ and $t$:$$J_W(v)=\int q_{0|t}(\boldsymbol{x}_0|\boldsymbol{x}_t) w(\boldsymbol{x}_0) \| v - \boldsymbol{v}_{t|0}(\boldsymbol{x}_t|\boldsymbol{x}_0) \|^2 \, d\boldsymbol{x}_0$$Expanding the quadratic term, we obtain:
\begin{equation}
    J_W(v)=\|v\|^2 \underbrace{\left( \int q_{0|t}(\boldsymbol{x}_0|\boldsymbol{x}_t) w(\boldsymbol{x}_0) \, d\boldsymbol{x}_0 \right)}_{M(\boldsymbol{x}_t)} - 2 v^T \underbrace{\left( \int q_{0|t}(\boldsymbol{x}_0|\boldsymbol{x}_t) w(\boldsymbol{x}_0) \boldsymbol{v}_{t|0}(\boldsymbol{x}_t|\boldsymbol{x}_0) \, d\boldsymbol{x}_0 \right)}_{N(\boldsymbol{x}_t)} + C\label{eq.quadratic_loss_extracted}
\end{equation}
Setting the gradient to zero yields $2v M(\boldsymbol{x}_t) - 2N(\boldsymbol{x}_t) = 0$. Unlike the unweighted case, the integral $M(\boldsymbol{x}_t)$ acts as a normalizing constant for the weights rather than summing to 1. Provided $M(\boldsymbol{x}_t) > 0$, we can solve for $v$ to find the optimal closed-form solution:
$$v^*=D_{\theta^*}(\boldsymbol{x}_t, t)=\frac{N(\boldsymbol{x}_t)}{M(\boldsymbol{x}_t)} = \frac{\mathbb{E}_{\boldsymbol{x}_0 \sim q_{0|t}} [ w(\boldsymbol{x}_0) \boldsymbol{v}_{t|0}(\boldsymbol{x}_t|\boldsymbol{x}_0) ]}{\mathbb{E}_{\boldsymbol{x}_0 \sim q_{0|t}} [ w(\boldsymbol{x}_0) ]}
$$
\end{proof}

\subsection{Velocity Field Generates Reweighted Data Distribution}\label{sec.apdx.generates_reweights_datadist}
Finally, we prove that $D_{\theta^*}$ generates the probability path:
$$
\rho_t(\boldsymbol{x}_t|\boldsymbol{s})=\frac{1}{Z} \int w(\boldsymbol{s},\boldsymbol{x}_0) p_{\text{data}}(\boldsymbol{x}_0|\boldsymbol{s}) q_{t|0}(\boldsymbol{x}_t|\boldsymbol{x}_0) \, d\boldsymbol{x}_0
$$
where $Z(\boldsymbol{s})=\int w(\boldsymbol{s},\boldsymbol{x}_0) p_{\text{data}}(\boldsymbol{x}_0|\boldsymbol{s}) \, d\boldsymbol{x}_0$ is the global normalizing constant.
\begin{proof}
To prove $D_{\theta^*}(\boldsymbol{x}_t, t)$ generates this distribution, we verify that it satisfies the continuity equation: $\frac{\partial \rho_t}{\partial t} + \nabla \cdot (\rho_t D_{\theta^*}) = 0$. Let $\tilde{\rho}_t(\boldsymbol{x}_t) = Z \rho_t(\boldsymbol{x}_t)$ denote the unnormalized density. The flow term is:
$$
\tilde{\rho}_t D_{\theta^*}=\int w(\boldsymbol{x}_0) p_{\text{data}}(\boldsymbol{x}_0) q_{t|0}(\boldsymbol{x}_t|\boldsymbol{x}_0) \boldsymbol{v}_{t|0}(\boldsymbol{x}_t|\boldsymbol{x}_0) \, d\boldsymbol{x}_0
$$
Applying the divergence operator and substituting the conditional continuity equation $\nabla \cdot (q_{t|0} \boldsymbol{v}_{t|0}) = - \frac{\partial q_{t|0}}{\partial t}$, we find:
$$
\nabla \cdot (\tilde{\rho}_t D_{\theta^*})=\int w(\boldsymbol{x}_0) p_{\text{data}}(\boldsymbol{x}_0) \left( - \frac{\partial q_{t|0}(\boldsymbol{x}_t|\boldsymbol{x}_0)}{\partial t} \right) \, d\boldsymbol{x}_0
$$
By the linearity of the derivative, we extract the time derivative:
$$
\nabla \cdot (\tilde{\rho}_t D_{\theta^*})=- \frac{\partial}{\partial t} \int w(\boldsymbol{x}_0) p_{\text{data}}(\boldsymbol{x}_0) q_{t|0}(\boldsymbol{x}_t|\boldsymbol{x}_0) \, d\boldsymbol{x}_0 = - \frac{\partial \tilde{\rho}_t}{\partial t}
$$
Thus, $\frac{\partial \tilde{\rho}_t}{\partial t} + \nabla \cdot (\tilde{\rho}_t v_t^*) = 0$, confirming that the vector field properly evolves $\rho_0$ to $\rho_1$.
\end{proof}

\section{Population-based Estimator for the Normalization Term in Linear and Square Reweighting}
\label{app:relu_derivations}

In this section, we detail on how to obtain the normalization term $\nu(\sbb)$ when using linear and square reweighting. 

\subsection{General Setup}
We omit the dependence on states $\sbb$. We assume we have access to a batch of samples $\{\ab_{(i)}\}_{i=1}^N$, and their $Q$-values are $\{Q(\ab_{(i)})\}_{i=1}^N$ accordingly. Denote $x_{(i)}=Q(\ab_{(i)})$ for brevity, and we assume $\{x_{(i)}\}_{i=1}^N$ are sorted in descending order such that $x_{(1)} \ge x_{(2)} \ge \dots \ge x_{(N)}$.

We define the \textit{active set} as $S = \{i \mid x_i > \nu\}$, i.e. the index set of samples which will not be clipped to zero. Let $k = |S|$ denote the number of active samples. For $i \in S$, $\bar{g}_f(x_i - \nu) = x_i - \nu$; otherwise, it is 0 according to linear and square reweighting. 

\paragraph{High-level Idea. }Because the $\max(\cdot, 0)$ operation truncates values below the threshold $\nu$ to zero, only the top $k$ elements in our sorted sequence will actively contribute to the target sum. The core challenge is determining this exact number of active elements, $k$, without knowing $\nu$ beforehand.

We achieve this by testing each possible $k \in \{1, \dots, N\}$ and calculating an "excess" metric. This metric measures the cumulative difference between the top $k$ elements and the candidate $k$-th element, $x_{(k)}$. By checking if this excess falls within our target budget (which is exactly $1$), we can identify the true number of active samples. Once $k$ is determined, the threshold $\nu$ can be analytically solved using the sums of those active elements.

\subsection{Case I: Linear Reweighting}
For the linear case, the normalization constraint requires that the sum of the active clipped values equals 1:$$\sum_{i=1}^k (x_{(i)} - \nu) = 1$$Assuming we know $k$, we can distribute the sum and solve directly for $\nu$:$$\sum_{i=1}^k x_{(i)} - k\nu = 1 \implies \nu = \frac{1}{k} \left( \sum_{i=1}^k x_{(i)} - 1 \right)$$To determine $k$, we use the definition of the active set: for the $k$-th element to be active, we must have $x_{(k)} > \nu$. Substituting our expression for $\nu$ into this inequality yields:$$x_{(k)} > \frac{1}{k} \left( \sum_{i=1}^k x_{(i)} - 1 \right)$$Rearranging the terms gives us our condition:$$k x_{(k)} > \sum_{i=1}^k x_{(i)} - 1 \implies \sum_{i=1}^k x_{(i)} - k x_{(k)} < 1$$The term $\sum_{i=1}^k x_{(i)} - k x_{(k)}$ is the "excess" if the threshold were set exactly at $x_{(k)}$. We compute this excess for all possible indices using the cumulative sum of our sorted values. The true number of active particles $k$ is simply the count of indices where this excess is less than or equal to 1. After identifying $k$, we plug it back into our initial equation to find the final normalization term $\nu$.



\subsection{Case II: Square Reweighting}
For the square case, the constraint requires the sum of the squared active values to equal 1:$$\sum_{i=1}^k (x_{(i)} - \nu)^2 = 1$$Expanding the square and distributing the sum gives a quadratic equation in terms of $\nu$:$$k\nu^2 - 2 \left( \sum_{i=1}^k x_{(i)} \right) \nu + \left( \sum_{i=1}^k x_{(i)}^2 - 1 \right) = 0$$Let $S_1 = \sum_{i=1}^k x_{(i)}$ and $S_2 = \sum_{i=1}^k x_{(i)}^2$. The equation simplifies to $k\nu^2 - 2S_1\nu + (S_2 - 1) = 0$. Using the quadratic formula and taking the smaller root (to guarantee that $\nu$ remains strictly less than the active elements), we obtain:$$\nu = \frac{S_1 - \sqrt{S_1^2 - k S_2 + k}}{k}$$Similar to the linear case, we must first find $k$. The condition $x_{(k)} > \nu$ is algebraically equivalent to requiring that the sum of squared differences from $x_{(k)}$ is less than or equal to our budget of 1:$$\sum_{i=1}^k (x_{(i)} - x_{(k)})^2 \le 1$$Expanding this excess metric gives:$$\sum_{i=1}^k x_{(i)}^2 - 2 x_{(k)} \sum_{i=1}^k x_{(i)} + k x_{(k)}^2 \le 1$$$$S_2 - 2 x_{(k)} S_1 + k x_{(k)}^2 \le 1$$We evaluate this condition for all indices using the cumulative sums of both $x$ and $x^2$. The active set size $k$ is the number of elements that satisfy this inequality. Finally, we use this $k$, along with its corresponding $S_1$ and $S_2$, to compute the discriminant $\Delta = S_1^2 - k S_2 + k$ and solve for $\nu$.



\subsection{Reference Implementation}
\label{app:jax_implementation}
The above calculation can be fully vectorized in our JAX implementation. In the below, we provide a vectorized JAX implementation for solving the normalization factor for both the linear and square reweighting. 

\begin{lstlisting}[
language=Python,
    caption={JAX implementation for solving $\nu$.},
    label={lst:jax_code},
    basicstyle=\ttfamily\footnotesize,
    backgroundcolor=\color{black!5},          % Subtle light gray background
    keywordstyle=\color{blue}\bfseries,       % Bold blue keywords
    commentstyle=\color{green!60!black}\itshape, % Italic green comments
    stringstyle=\color{purple},               % Purple strings
    numbers=left,                             % Line numbers on the left
    numberstyle=\tiny\color{gray},            % Tiny gray line numbers
    stepnumber=1,
    numbersep=8pt,                            % Space between numbers and code
    frame=single,
    rulecolor=\color{black!20},               % Soft gray border
    breaklines=true,
    showstringspaces=false,
    showspaces=false,
    xleftmargin=2em,                          % Margin for line numbers
    framexleftmargin=1.5em,                   % Extends background behind numbers
    xrightmargin=0.5em,
    columns=flexible,
    keepspaces=true
]
def solve_normalizer_linear(q: jnp.ndarray, temp: float):
    num_particles = q.shape[-1]
    q_sorted = jnp.sort(q, axis=-1)[..., ::-1]
    q_cumsum = jnp.cumsum(q_sorted, axis=-1)

    k_indices = jnp.arange(1, num_particles+1).reshape(1, -1)
    excess = q_cumsum - k_indices * q_sorted
    active_mask = excess <= temp
    k = jnp.sum(active_mask, axis=-1, keepdims=True)

    sum_active = jnp.take_along_axis(q_cumsum, k-1, axis=-1)
    nu = (sum_active - temp) / k
    return nu

def solve_normalizer_square(q: jnp.ndarray, temp: float):
    q_max = jnp.max(q, axis=-1, keepdims=True)
    q = q - q_max
    num_particles = q.shape[-1]
    target_sum = temp ** 2

    q_sorted = jnp.sort(q, axis=-1)[..., ::-1]
    q_cumsum = jnp.cumsum(q_sorted, axis=-1)
    q_squared_cumsum = jnp.cumsum(q_sorted ** 2, axis=-1)

    k_indices = jnp.arange(1, num_particles+1).reshape(1, -1)
    excess = q_squared_cumsum \
            - 2 * q_cumsum * q_sorted \
            + k_indices * (q_sorted ** 2)
    active_mask = excess <= target_sum
    k = jnp.maximum(jnp.sum(active_mask, axis=-1, keepdims=True), 1)

    S1 = jnp.take_along_axis(q_cumsum, k-1, axis=-1)
    S2 = jnp.take_along_axis(q_squared_cumsum, k-1, axis=-1)
    delta = S1**2 - k * S2 + k * target_sum
    nu = (S1 - jnp.sqrt(jnp.maximum(delta, 0.0))) / k
    return nu + q_max
\end{lstlisting}

\section{Auto-tuning the Temperature}
\label{apdx.temp_tuning}

The temperature $\lambda$ in \algbb controls the sharpness of the reweighting distribution over the $N$ candidate actions $\{a_i\}_{i=1}^N$ drawn from the old policy. Its appropriate scale is tightly coupled with the magnitude of the $Q$-values, which drifts substantially during training, making a fixed $\lambda$ brittle. We therefore adapt $\lambda$ online by constraining the divergence between the reweighting distribution and the uniform distribution over particles. Let $w_i(\lambda)$ denote the normalized per-particle weights produced by the chosen reweighting scheme (exponential, linear, or squared), so that $\sum_i w_i = 1$. Its Shannon entropy is $\mathcal{H}(w) = -\sum_i w_i \log w_i$, and its KL divergence to the uniform distribution $\mathrm{Unif}_N$ satisfies
\begin{equation}                                 
  \mathrm{KL}\!\left(w \,\|\, \mathrm{Unif}_N\right) \;=\; \log N - \mathcal{H}(w).
\end{equation}

Following the dual formulation of SAC~\citep{haarnoja2018soft}, we treat $\lambda$ as the Lagrange multiplier of the soft constraint $\mathrm{KL}(w \,\|\, \mathrm{Unif}_N) \le \kappa$, where $\kappa$ is a target KL budget. Dual ascent yields the surrogate temperature loss                       
\begin{equation}
  \mathcal{L}_\lambda \;=\; \lambda \cdot \bigl(\kappa + \mathcal{H}(w) - \log N\bigr) \;=\; \lambda \cdot \bigl(\kappa - \mathrm{KL}(w \,\|\, \mathrm{Unif}_N)\bigr),           
\end{equation}
minimized w.r.t.\ $\lambda$ jointly with the actor and critic. Intuitively, when the reweighting is too peaky ($\mathrm{KL}>\kappa$) the gradient drives $\lambda$ upward, smoothing the weights; when it is too diffuse ($\mathrm{KL}<\kappa$) it drives $\lambda$ downward, sharpening them. Positivity of $\lambda$ is enforced via a softplus parameterization.

\section{Matching Policy with the Target Measure via Noise Contrastive Estimation}
\label{app.nce}

As an alternative to reweighted flow matching, we can align the current diffusion policy $\pi_\theta$ with the target policy using Noise Contrastive Estimation (NCE). NCE frames density estimation as a classification problem, distinguishing between \textit{positive} samples from the target policy and \textit{negative} samples from a noise proposal distribution. In our setting, the target policy is $\pi^*(\ab_0|\sbb) \propto\pi_{\rm old}(\ab_0|\sbb)w(\sbb, \ab_0)$, where $w(\sbb, \ab_0) = g_f\left(\frac{Q(\sbb, \ab_0) - \nu(\sbb)}{\lambda}\right)$ is the reweighting factor. We naturally select the old policy $\pi_{\rm old}(\ab_0|\sbb)$ as our noise proposal distribution.

Consider a set of $M$ actions $\{\ab^{(1)}_0, \ab^{(2)}_0, \dots, \ab^{(M)}_0\}$ sampled independently from the old policy $\pi_{\rm old}(\cdot|\sbb)$ for a given state $\sbb$. Under the target policy, the categorical probability of selecting a specific action $\ab^{(i)}_0$ from this set is independent of the global partition function $Z(\sbb)$, as it cancels out in the normalization:
$$
p^*(\ab^{(i)}_0 | \{\boldsymbol{a}_0^{(j)}\}_{j=1}^M, \sbb) = \frac{\pi^*(\ab^{(i)}_0|\sbb) / \pi_{\rm old}(\ab^{(i)}_0|\sbb)}{\sum_{j=1}^M \pi^*(\ab^{(j)}_0|\sbb) / \pi_{\rm old}(\ab^{(j)}_0|\sbb)} = \frac{w(\sbb, \ab^{(i)}_0)}{\sum_{j=1}^M w(\sbb, \ab^{(j)}_0)}
$$
We construct the categorical probability of select action $\boldsymbol{a}_0^{(i)}$ from this set as:
$$
p_\theta(\boldsymbol{a}_0^{(i)}|\{\boldsymbol{a}_0^{(j)}\}_{j=1}^M, \sbb)=\frac{ \exp\left(-\frac 1\gamma\mathbb{E}_{t, \boldsymbol{a}_t}[\|D_\theta(\boldsymbol{s}, \boldsymbol{a}_t, t) - \boldsymbol{v}_{t|0}(\ab_t|\ab_0^{(i)})\|^2]\right)}{\sum_{j=1}^M \exp\left(-\frac 1\gamma\mathbb{E}_{t, \boldsymbol{a}_t}[\|D_\theta(\boldsymbol{s}, \boldsymbol{a}_t, t) - \boldsymbol{v}_{t|0}(\ab_t|\ab_0^{(j)})\|^2]\right)}
$$

To align the diffusion policy $\pi_\theta$ with the target measure, we minimize the cross-entropy between the target categorical distribution $p^*$ and the current policy's categorical distribution $p_\theta$ over the sampled set:
$$
\begin{aligned}
\min_\theta\ \mathcal{L}(\theta) &= - \mathbb{E}_{\sbb, \ab_0^{(1):(M)}} \left[ \sum_{i=1}^M p^*(\ab^{(i)}_0 | \{\boldsymbol{a}_0^{(j)}\}_{j=1}^M, \sbb) \log p_\theta(\ab^{(i)}_0 | \{\boldsymbol{a}_0^{(j)}\}_{j=1}^M, \sbb) \right]\\
&=\mathbb{E}_{\sbb, \ab_0^{(1):(M)}} \Bigg[ \sum_{i=1}^M \frac{w(\sbb, \ab_0^{(i)})}{\sum_j w(\sbb, \ab_0^{(j)})}\frac 1\gamma\mathbb{E}_{t, \boldsymbol{a}_t}[\|D_\theta(\boldsymbol{s}, \boldsymbol{a}_t, t) - \boldsymbol{v}_{t|0}(\ab_t|\ab_0^{(i)})\|^2]\\
&\quad\quad\quad\quad+\log\sum_{j=1}^M\exp\left(-\frac 1\gamma\mathbb{E}_{t, \boldsymbol{a}_t}[\|D_\theta(\boldsymbol{s}, \boldsymbol{a}_t, t) - \boldsymbol{v}_{t|0}(\ab_t|\ab_0^{(j)})\|^2]\right)\Bigg]\\
&\leq\mathbb{E}_{\sbb, \ab_0^{(1):(M)}, \ab_t^{(1):(M)}} \Bigg[ \underbrace{\sum_{i=1}^M \frac{w(\sbb, \ab_0^{(i)})}{\sum_j w(\sbb, \ab_0^{(j)})}\frac 1\gamma\|D_\theta(\boldsymbol{s}, \boldsymbol{a}^{(i)}_t, t) - \boldsymbol{v}_{t|0}(\ab^{(i)}_t|\ab_0^{(i)})\|^2}_{\text{positive part}}\\
&\quad\quad\quad\quad+\underbrace{\log\sum_{j=1}^M\exp\left(-\frac 1\gamma\|D_\theta(\boldsymbol{s}, \boldsymbol{a}^{(j)}_t, t) - \boldsymbol{v}_{t|0}(\ab^{(j)}_t|\ab_0^{(j)})\|^2\right)}_{\text{negative part}}\Bigg],\\
\end{aligned}
$$
where the final inequality follows from Jensen’s inequality and the convexity of the LogSumExp function. The resulting objective consists of two components: (1) a weighted flow matching loss that minimizes the flow matching error, thereby increasing the likelihood of the policy generating high-value actions; and (2) a negative part that penalizes the likelihood of alternative samples, effectively pushing the policy's density away from undesired regions of the action space.


\section{Experimental Details}
\label{sec.apdx.experimental_details}

\subsection{Baselines}
\label{sec:apdx_baselines}

We evaluate our approach against two distinct families of reinforcement learning (RL) algorithms. The first group consists of five state-of-the-art \textbf{online diffusion-policy methods}, encompassing both off-policy and on-policy architectures:

\begin{itemize}
    \item \textbf{QSM}~\cite{psenka2023learning}: Implements Langevin dynamics using the gradient of a learned $Q$-function as the score function.
    \item \textbf{QVPO}~\cite{ding2024diffusion}: Employs a $Q$-weighted variational objective for diffusion training, though it is notably constrained by its inability to process negative rewards effectively.
    \item \textbf{DACER}~\cite{wang2024diffusion}: Backpropagates gradients directly through the reverse diffusion process and utilizes a GMM entropy regulator to manage the exploration-exploitation trade-off.
    \item \textbf{DIPO}~\cite{yang2023policy}: Uses a two-stage approach where state-action particles are updated via $Q$-gradients before being fitted to a diffusion model.
\end{itemize}

The second group includes three \textbf{classic model-free baselines}: \textbf{TD3}~\cite{fujimoto2018addressing}, and \textbf{SAC}~\cite{haarnoja2018soft}.

\begin{table}[h]
    \centering
    \caption{Hyperparameters}
    \begin{tabular}{l|c}
    \toprule
      \textbf{Name}   & \textbf{Value} \\
      \midrule
      Critic learning rate & 3e-4 \\
      Policy learning rate & 1e-4\\
      Diffusion steps & 20\\
      Diffusion noise schedules & Cosine \\      
      Policy network hidden layers & 2\\
      Policy network hidden neurons & 256\\
      Policy network activation & Mish\\
      Value network hidden layers & 2\\
      Value network hidden neurons & 256\\
      Value network activation & ReLU\\
      replay buffer size (off-policy only) & 1 million\\
      \midrule
    Number of sampled particles  & 64 \\
    Additive Noise & 0.2 \\
    Diffusion Steps & 20 \\
    target KL budget $\kappa$ - Hopper-v5 & 2.5 \\
    target KL budget $\kappa$ - Ant-v5 & 1.5 \\
    target KL budget $\kappa$ - Other tasks & 2.0 \\
         \bottomrule
    \end{tabular}
    \label{tab:main_hyper}
\end{table}

For \textbf{\algbb-Linear}, \textbf{\algbb-Square} and \textbf{\algbb-Exp}, we implement the automatic temperature tuning detailed in \Cref{apdx.temp_tuning}, and sweep the KL tolerance within $\{1.5, 2.0, 2.5\}$. 


\subsection{Reward Functions}
\label{sec.apdx.reward}
The Cheetah domain tasks are designed to track a target horizontal velocity $v^* = 10.0$ m/s (limited by the physical capability, the robot cannot actually achieve 10m/s. Therefore, the task is to run as fast as possible.). Let $v$ denote the current horizontal velocity of the torso. We evaluate our method across several reward formulations, categorized into linear, quadratic, and absolute value variations.

We design the following reward functions (with proper transformation to align the scale), ranked from ``flat'' landscape to ``steep'' landscape,

\begin{itemize}
    \item \texttt{Sqrt} reward
    $R_{abs\_sqrt}(v) = \sqrt{|v - v^*|}$
    \item \texttt{Linear} reward $R_{\rm linear}(v) = \frac{v}{v^*}$
    \item \texttt{Square} reward
    $R_{abs\_sq}(v) = |v - v^*|^2$
    \item \texttt{Exponental} reward
     $R_{abs\_exp}(v) = \exp(-|v - v^*|)$
\end{itemize}

\subsection{Compute Resources and Runtime}
\label{sec.apdx.resource}
We run our experiments on a workstation with 8 RTX 5090s. On our workstation, a single run of bandit problem takes around 5 minutes, a single run of MuJoCo/DMC task takes around 2 hours, while a single run of DNA model fine-tuning takes around 35 minutes (estimation might be effected by other users' activities).
No excessive computational capabilities are required.

\section{Detailed Explaination of the Repelling Effects of Negative Reweighting}
\label{sec.geometric_interp}
 We provide a detailed explaination about the repelling effect here.

Compared to original optimal solution with non-negative weights in~\Cref{eq.closed_form_solution_reweighting}, there are two cases if we consider weights being possibly negative, depending on the sign of denominator $\mathbb{E}_{\ab_0 \sim p_{0|t}} [ w(\sbb,\ab_0) ]$:

\textbf{Case I~(discussed in~\Cref{sec.negative-weighting}):} For all $\ab_t$ with $\mathbb{E}_{\ab_0 \sim p_{0|t}} [ w(\sbb,\ab_0) ] > 0$, 
\Cref{eq.closed_form_solution_reweighting} still servers as the optimal solution to~\eqref{eq:weighted_matching}. However, the negative weights will change the sign of $\boldsymbol{v}_{t|0}$ in the weighted average, pushing the weighted average velocity away from those $\boldsymbol{v}_{t|0}$ with negative weights.


\begin{remark}[Another interpretation from generation of reweighted distributions via signed measures]
\label{rem:signed_measure_generation}
Even when the weights $w(\sbb, \ab_0)$ take negative values (as in \textbf{Case I}), the closed-form velocity field solution from \Cref{eq.closed_form_solution_reweighting} generates a valid probability path $\rho_t$, provided the total measure remains positive. The generated path is given by:
\begin{equation}
    \rho_t(\ab|\sbb) = \frac{1}{Z(\sbb)} \int w(\sbb,\ab_0) \pi_{\rm old}(\ab_0|\sbb) p_{t|0}(\ab|\ab_0) \, d\ab_0,
\end{equation}
where $Z(\sbb) = \int w(\sbb, \ab_0) \pi_{\rm old}(\ab_0|\sbb) d\ab_0$ is the normalization constant. Mathematically, the negative weights introduce negative components to the integral. To satisfy the normalization constraint $\int \rho_t(\ab|\sbb) d\ab = 1$, the contribution from regions with positive weights must be amplified to offset the negative mass. Consequently, this mechanism effectively concentrates significantly higher probability mass on positive samples than standard reweighting.
\end{remark}

\textbf{Case II:} For all $\ab_t$ with $\mathbb{E}_{\ab_0 \sim p_{0|t}} [ w(\sbb,\ab_0) ] \leq 0$, the quadratic term in loss function in~\eqref{eq.quadratic_loss_extracted} has negative coefficients. Therefore, the loss is unbounded and have a ``repelling'' effect, consistently pushing $D_\theta$ away from the solution in~\eqref{eq.closed_form_solution_reweighting} to arbitrary directions, which also pushes generated path $\ab_t$ away from the negative weights region.

In summary, in both cases the negative weights will \emph{push the generated probability paths away from actions with negative weights}. 

\textbf{Repelling effect helps policy improvement.} There are two possible outcomes from the repelling effect: \textbf{a)} pushing the actions towards higher advantages such as Remark~\ref{rem:signed_measure_generation}; \textbf{b)} pushing actions to regions not included in the support of old policy $\pi_{\rm old}$, encouraging exploration. Both cases provide useful feedback from negative samples.




\newpage

\end{document}